\documentclass[11pt]{article}
\usepackage[in]{fullpage}
\usepackage[parfill]{parskip}
\usepackage{amsmath,amsthm,amssymb,bbm}
\usepackage{mathtools}
\usepackage{cases}
\usepackage{dsfont}
\usepackage{microtype}
\usepackage{subfigure}
\usepackage{algorithm,algorithmic}
\usepackage{color}
\usepackage{appendix}

\usepackage{bm}
\usepackage{subfigure}
\usepackage{algorithm,algorithmic}
\usepackage{color}
\usepackage{booktabs}       
\usepackage{appendix}
\usepackage{authblk}

\usepackage{url}
\usepackage[authoryear]{natbib}
\usepackage[colorlinks,citecolor=blue,urlcolor=blue,linkcolor=blue,linktocpage=true]{hyperref}
\usepackage{cleveref}
\crefformat{equation}{(#2#1#3)}
\crefrangeformat{equation}{(#3#1#4) to~(#5#2#6)}
\crefname{equation}{}{}
\Crefname{equation}{}{}

\crefname{definition}{\textbf{definition}}{definitions}
\Crefname{definition}{Definition}{Definitions}
\crefname{assumption}{\textbf{assumption}}{assumptions}
\Crefname{assumption}{Assumption}{Assumptions}

\definecolor{maroon}{RGB}{192,80,77}
\definecolor{mypink3}{cmyk}{0, 0.7808, 0.4429, 0.1412}
\newcommand{\maroon}[1]{\textcolor{maroon}{#1}}

\newtheorem{theorem}{Theorem}[section]
\newtheorem{lemma}[theorem]{Lemma}

\newtheorem{corollary}[theorem]{Corollary}
\newtheorem{definition}[theorem]{Definition}
\newtheorem{question}[theorem]{Question}

\newtheorem{remark}[theorem]{Remark}

\usepackage{amsmath}

\usepackage{tikz}

\begin{document}

\title{Sample-Efficient Reinforcement Learning\\ with loglog(T) Switching Cost}
\author[1]{Dan Qiao}
\author[1,2]{Ming Yin}
\author[2]{Ming Min}
\author[1]{Yu-Xiang Wang}
\affil[1]{Department of Computer Science, UC Santa Barbara}
\affil[2]{Department of Statistics and Applied Probability, UC Santa Barbara}
\affil[ ]{\texttt{\{danqiao,ming\_yin,m\_min\}@ucsb.edu}, \;
	\texttt{yuxiangw@cs.ucsb.edu}}

\date{}

\maketitle

\begin{abstract}
 We study the problem of reinforcement learning (RL) with low (policy) switching cost — a problem well-motivated by real-life RL applications in which deployments of new policies are costly and the number of policy updates must be low. In this paper, we propose a new algorithm based on stage-wise exploration and adaptive policy elimination that achieves a regret of $\widetilde{O}(\sqrt{H^4S^2AT})$ while requiring a switching cost of $O(HSA \log\log T)$. This is an exponential improvement over the best-known switching cost $O(H^2SA\log T)$ among existing methods with $\widetilde{O}(\mathrm{poly}(H,S,A)\sqrt{T})$ regret. In the above, $S,A$ denotes the number of states and actions in an $H$-horizon episodic Markov Decision Process model with unknown transitions, and $T$ is the number of steps.
 As a byproduct of our new techniques, we also derive a reward-free exploration algorithm with a switching cost of $O(HSA)$.
 Furthermore, we prove a pair of information-theoretical lower bounds which say that 
 (1) Any no-regret algorithm must have a switching cost of $\Omega(HSA)$; (2) Any  $\widetilde{O}(\sqrt{T})$ regret algorithm must incur a switching cost of $\Omega(HSA\log\log T)$. Both our algorithms are thus optimal in their switching costs.
\end{abstract}

\newpage
\tableofcontents
\newpage


\section{Introduction}\label{sec:introduction}
In many real-world reinforcement learning (RL) tasks, it is costly to run fully adaptive algorithms that update the exploration policy frequently. Instead, collecting data in large batches using the current policy deployment is usually cheaper. 
For instance, in recommendation systems \citep{afsar2021reinforcement}, the system is able to collect millions of new data points in minutes, while the deployment of a new policy often takes weeks, as it involves significant extra cost and human effort. It is thus infeasible to change the policy after collecting every new data point as a typical RL algorithm would demand. A practical alternative is to schedule a large batch of experiments in parallel and only decide whether to change the policy after the whole batch is complete. Similar constraints arise in other RL applications such as those in healthcare \citep{yu2021reinforcement}, database optimization \citep{krishnan2018learning}, computer networking \citep{xu2018experience} and new material design \citep{raccuglia2016machine}.

In those scenarios, the agent needs to minimize the number of policy switching while maintaining (nearly) the same regret bounds as its fully-adaptive counterparts. On the empirical side, \citet{matsushima2020deployment} cast this problem via the notion \emph{deployment efficiency} and designed algorithms with high deployment efficiency for both online and offline tasks. On the theoretical side, \citet{bai2019provably} first brought up the definition of \emph{switching cost} that measures the number of policy updates. They designed $Q$ learning-based algorithm with regret of $\widetilde{O}^{*}(\sqrt{T})$ and switching cost of $O^{*}(\log T)$\footnote{Here $O^*(\cdot)$ and $\widetilde{O}^{*}$ omit a $\mathrm{poly}(H,S,A)$ terms, this will be a notation we use throughout.}. Later, \citet{zhang2020almost} improved both the regret bound and switching cost bound. However, the switching cost remains order $O^{*}(\log T)$. In addition, both algorithms need to monitor the data stream to decide whether the policy should be switched at each episode. In contrast, for an $A$-armed bandit problem, \citet{cesa2013online} created \emph{arm elimination} algorithm that achieves the optimal $\widetilde{O}(\sqrt{AT})$ regret
and a near constant switching cost bound of $O(A\log\log T)$. Meanwhile, the arm elimination algorithm predefined when to change policy before the algorithm starts, which could render parallel implementation. To adapt this feature from multi-armed bandit to RL problem, one straightforward way is to consider each deterministic policy ($A^{SH}$ policies in total) as an arm. Applying the same algorithm for the RL setting, one ends up with the switching cost to be $O(A^{SH}\log\log T)$ and the regret bound of order $O(\sqrt{A^{SH}T})$. Clearly, such an adaptation is far from satisfactory as the exponential dependence on $H,S$ makes the algorithm inefficient. 
This motivates us to consider the following question:
\begin{table*}[!t]\label{tab:comparison}
\centering
\resizebox{\linewidth}{!}{
\begin{tabular}{ |c|c|c| } 
\hline
\textit{Algorithms for regret minimization} & \textit{Regret}  & \textit{Switching cost} \\
\hline 
UCB2-Bernstein \citep{bai2019provably} & $\widetilde{O}(\sqrt{H^{3}SAT})$   & Local: $O(H^{3}SA\log T)$\\ 
UCB-Advantage \citep{zhang2020almost} & $\widetilde{O}(\sqrt{H^{2}SAT})$  & Local: $O(H^{2}SA\log T)$ \\ 
Algorithm 1 in \citep{gao2021provably} $^*$ & $\widetilde{O}(\sqrt{d^{3}H^{3}T})$  & Global: $O(dH\log T)$ \\  
\textcolor{blue}{APEVE (Our Algorithm~\ref{algo3})} & \textcolor{blue}{$\widetilde{O}(\sqrt{H^{4}S^{2}AT})$}  & \textcolor{blue}{Global: $O(HSA\log\log T)$} \\
\textcolor{blue}{Explore-First w. LARFE  (Our Algorithm \ref{algo4})} & \textcolor{blue}{$\widetilde{O}(T^{2/3}H^{4/3}S^{2/3}A^{1/3})$}  & \textcolor{blue}{Global: $O(HSA)$} \\
\hline
\textcolor{blue}{Lower bound (Our Theorem~\ref{thm:lower_bound})} & \textcolor{blue}{if $\widetilde{O}(\sqrt{T})$ (``Optimal regret'')} & \textcolor{blue}{Global: $\Omega(HSA\log\log T)$}    \\
\textcolor{blue}{Lower bound (Our Theorem~\ref{the2})} & \textcolor{blue}{if $o(T)$ (``No regret'')} & \textcolor{blue}{Global: $\Omega(HSA)$}    \\
\hline
\hline
\textit{Algorithms for reward-free exploration} & \textit{Sample (episode) complexity}  & \textit{Switching cost}  \\
\hline
Algorithm 2$\&$3 in \citep{jin2020reward} & $\widetilde{O}(\frac{H^{5}S^{2}A}{\epsilon^{2}})$  & Global: $\widetilde{O}(\frac{H^{7}S^{4}A}{\epsilon})^\ddagger$ \\  
RF-UCRL \citep{kaufmann2021adaptive} & $\widetilde{O}(\frac{H^{4}S^{2}A}{\epsilon^{2}})$  & Global: $\widetilde{O}(\frac{H^{4}S^{2}A}{\epsilon^{2}})$ \\ 
RF-Express \citep{menard2021fast} & $\widetilde{O}(\frac{H^{3}S^{2}A}{\epsilon^{2}})$  & Global: $\widetilde{O}(\frac{H^{3}S^{2}A}{\epsilon^{2}})$ \\ 
SSTP \citep{zhang2020nearly} & $\widetilde{O}(\frac{S^{2}A}{\epsilon^{2}})^{\star}$ & Global: $\widetilde{O}(SA\log(\frac{S^{2}A}{\epsilon^{2}}))^{\dagger}$ \\  
Algorithm 3$\&$4 in \citep{huang2022towards} & $\widetilde{O}(dH(\frac{d^{3c_{K}}H^{6c_{K}+1}}{\epsilon^{2c_{K}}})^{\frac{1}{c_{K}-1}})$ & Global: $c_{K}dH+1$ \\
\textcolor{blue}{LARFE (Our Algorithm~\ref{algo4})} & \textcolor{blue}{$\widetilde{O}(\frac{H^{5}S^{2}A}{\epsilon^{2}})$} & \textcolor{blue}{Global: $O(HSA)$} \\
\hline 
\end{tabular}
}
\caption{
Comparison of our results (in \textcolor{blue}{blue}) to existing work regarding problem type, regret/sample complexity, and switching cost. Note that some of the works are under linear MDP, where $d$ is the dimension of feature map. When the feature map is the canonical basis \citep{jin2020provably}, linear MDP recovers tabular MDP and $d=SA$. $*$:This result is generalized by \citet{wang2021provably}, whose algorithm has a same switching cost bound under this regret bound. $\ddagger$: In \citep{jin2020reward}, there are $\widetilde{O}(\frac{H^{7}S^{4}A}{\epsilon})$ episodes of data collected using EULER, which can lead to the same number of switching cost in the worst case.  $\star$: This result is derived under stationary MDP with total reward bounded by $1$. $\dagger$:We translate the use of trigger set in Algorithm 3  \citep{zhang2020nearly} to a worst case switching cost bound. }
\end{table*}

\vspace{-0.5em}
\begin{question}\label{ques1}
Is it possible to design an algorithm for online RL problem with $O^*(\log\log T)$
switching cost and $\widetilde{O}(\mathrm{poly}(H,S,A)\sqrt{T})$ regret bound
while it can decide when to change policy before the process starts?
\end{question}

\noindent\textbf{Our contributions.} In this paper, we 
answer the above question affirmatively by contributing the new low switching algorithm APEVE (Algorithm~\ref{algo3}). Furthermore, the framework of APEVE naturally adapts to the more challenging \emph{low switching reward-free} setting and we end up with LARFE (Algorithm~\ref{algo4}) as a byproduct. Our concrete contributions are summarized as follows. To the best of our knowledge, all of the results are the first of its kinds.




\begin{itemize}
\itemsep0em
\item A new policy elimination algorithm APEVE (Algorithm~\ref{algo3}) that achieves $O(HSA\log \log T)$ switching costs (Theorem~\ref{the1}). This provides an exponential improvement over the existing algorithms that require an $O^*(\log T)$ switching cost to achieve $\widetilde{O}^*(\sqrt{T})$ regret.\footnote{To be rigorous, we point out there are different notions for switching cost, \emph{e.g.} local switching cost \citep{zhang2020almost} and global switching cost (ours). However, we are the first to achieve $\log\log T$ switching cost with $\sqrt{T}$ regret, regardless of its type.}

\item A matching global switching cost lower bound of $\Omega(HSA\log\log T)$ for any algorithm with $\widetilde{O}^*(\sqrt{T})$ regret bound (Theorem~\ref{thm:lower_bound}). This certifies the policy switching of APEVE is near-optimal for sample-efficient RL. As a byproduct, we provide a global switching cost lower bound of $\Omega(HSA)$ for any no-regret algorithm (Theorem~\ref{the2}).  

\item We also propose a new low-adaptive algorithm LARFE for \emph{reward-free} exploration (Algorithm~\ref{algo4}). It comes with an optimal global switching cost of $O(HSA)$ for  deterministic policies (Theorem~\ref{the3}) and allows the identification of an $\epsilon$-optimal policy \emph{simultaneously} for all (unknown, possibly data-dependent) reward design.
\end{itemize}

\noindent\textbf{Why $\log\log T$ switching cost matters?} The improvement from $\log T$ to $\log\log T$ could make a big difference in practical applications. Take $T=1e5$ as an example,  $\log T \approx 11.5$ and $\log\log T \approx 2.4$. This represents a nearly 5x improvement in a reasonably-sized exploration dataset one can collect. 80\% savings in the required resources could distinguish between what is practical and what is not, and will certainly allow for more iterations. On the other hand, the total number of atoms in the observable universe $\approx 10^{82}$ and $\log\log 10^{82}\approx 5.24$. This reveals $\log\log T$ could be cast as \emph{constant} quantity in practice, since it is impossible to run $T>10^{82}$ steps for any experiment in real-world applications.


\noindent\textbf{Related work.} There is a large and growing body of literature on the statistical theory of reinforcement learning that we will not attempt to thoroughly review. Detailed comparisons with existing work on RL with low-switching cost \citep{bai2019provably,zhang2020almost,gao2021provably,wang2021provably} and reward-free exploration \citep{jin2020reward,kaufmann2021adaptive,menard2021fast,zhang2020nearly,huang2022towards} are given in Table~\ref{tab:comparison}. For a slightly more general context of this work, please refer to Appendix~\ref{appr} and the references therein. Notably, all existing algorithms with a $\widetilde{O}^{*}(\sqrt{T})$ regret incurs a switching cost of $O^{*}(\log T)$.  In terms of lower bounds, our $\Omega(HSA)$ lower bound is stronger than that of \citet{bai2019provably} as it operates on the \emph{global} switching cost rather than the \emph{local} switching cost. 

The only existing algorithm with $o(\log T)$ switching cost comes from the concurrent work of \citet{huang2022towards} who studied the problem of \emph{deployment-efficient} reinforcement learning under the linear MDP model, where they require a constant switching cost. 
\citet{huang2022towards} obtained only sample complexity bounds for pure exploration, which makes their result incompatible to our regret bounds. When compared with our results in the reward-free RL setting in the tabular setting (taking $d=SA$) their algorithm has a comparable $O(HSA)$ switching cost, but incurs a larger sample complexity in $H,S,A$ and $\epsilon$. 



Lastly, the low-switching cost setting is often confused with its cousin --- the low adaptivity setting \citep{perchet2016batched,gao2019batched} (also known as \emph{batched} RL\footnote{Note that this is different from Batch RL, which is synonymous to Offline RL.}). Low-adaptivity requires decisions about policy changes to be made at only a few (often predefined) checkpoints but does not constrain the number of policy changes. Low-adaptive algorithms often do have low-switching cost, but lower bounds on rounds of adaptivity do not imply lower bounds for our problem. We note that our algorithms are low-adaptive, because they schedule the batch sizes of each policy ahead of time and require no adaptivity during the batches. This feature makes our algorithm more practical relative to \citep{bai2019provably,zhang2020almost,gao2021provably,wang2021provably} which uses adaptive switching \citep[see, e.g.,][for a more elaborate discussion]{huang2022towards}. In Section~\ref{sec:result} we will revisit this problem and highlight the optimality of our algorithm in this alternative setting, as a byproduct.


\noindent\textbf{A remark on technical novelty.} 
The design of our algorithms involves substantial technical innovation over \citet{bai2019provably,zhang2020almost,gao2021provably,wang2021provably}.  The common idea behind these $O(\log T)$ switching cost algorithms is the doubling schedule of batches in updating the policies, which originates from the UCB2 algorithm \citep{auer2002finite} for bandits. The change from UCB to UCB2 is mild enough such that existing ``optimism''-based algorithms for strategic exploration designed without switching cost constraints can be adapted. In contrast, algorithms with $O(\log \log T)$ switching cost deviates from ``optimism'' even in bandits problem \citep{cesa2013online}, thus require fresh new ideas in solving exploration when extended to RL.  

The generalization of the arm elimination schedule for bandits \citep{cesa2013online} to RL is nontrivial because there is an exponentially large set of deterministic policies but we need a sample efficient algorithm with polynomial dependence on $H,S,A$ (also see the discussion before Question~\ref{ques1}). Part of our solution is inspired by the reward-free exploration approach \citep{jin2020reward}, which learns to visit each $(h,s,a)$ as much as possible by designing special rewards. However, this approach itself requires an exploration oracle, and no existing RL algorithms has $o(\log T)$ switching cost (otherwise our problem is solved). We address this problem by breaking up the exploration into stages and iteratively update a carefully constructed “absorbing MDP” that can be estimated with multiplicative error bounds. Finally, our lower bound construction is new and simple, as it essentially shows that tabular MDPs are as hard as multi-armed bandits with $\Omega(HSA)$ arms in terms of the switching cost.  These techniques might be of independent interests beyond the context of this paper.

\section{Problem Setup }\label{sec:formulation}

\noindent\textbf{Episodic reinforcement learning.} We consider finite-horizon episodic \emph{Markov Decision Processes} (MDP) with non-stationary transitions. The model is defined by a tuple $\mathcal{M}=\langle\mathcal{S}, \mathcal{A}, P, r, H, d_1\rangle$ \citep{sutton1998reinforcement}, where $\mathcal{S}\times\mathcal{A}$ is the discrete state-action space and  $S:=|\mathcal{S}|,A:=|\mathcal{A}|$ are finite. A non-stationary transition kernel has the form $P:\mathcal{S}\times\mathcal{A}\times\mathcal{S}\times [H] \mapsto [0, 1]$  with $P_{h}(s^{\prime}|s,a)$ representing the probability of transition from state $s$, action $a$ to next state $s'$ at time step $h$. In addition, $r$ is a known\footnote{This is due to the fact that the uncertainty of reward function is dominated by that of transition kernel in RL.} expected (immediate) reward function which satisfies $r_{h}(s,a)\in [0,1]$. $H$ is the length of the horizon and $d_1$ is the initial state distribution. In this work, we assume there is a fixed initial state $s_{1}$.\footnote{The generalized case where $d_{1}$ is an arbitrary distribution can be recovered from this setting by adding one layer to the MDP.} A policy can be seen as a series of mapping $\pi=(\pi_1,...,\pi_H)$, where each $\pi_h$ maps each state $s \in \mathcal{S}$ to a probability distribution over actions, \emph{i.e.} $\pi_h: \mathcal{S}\rightarrow \Delta(\mathcal{A})$ ,where $\Delta(\mathcal{A})$ is the set of probability distributions over the actions, $\forall\, h\in[H]$. A random trajectory $ (s_1, a_1, r_1, \ldots, s_H,a_H,r_H,s_{H+1})$ is generated by the following rule: $s_1$ is fixed, $a_h \sim \pi_h(\cdot|s_h), r_h = r(s_h, a_h), s_{h+1} \sim P (\cdot|s_h, a_h), \forall\, h \in [H]$. 

\noindent\textbf{$Q$-values, Bellman (optimality) equations.} Given a policy $\pi$ and any $h\in[H]$, the value function $V^\pi_h(\cdot)\in \mathbb{R}^{S}$ and Q-value function $Q^\pi_h(\cdot,\cdot)\in \mathbb{R}^{S\times A}$ are defined as:
$
V^\pi_h(s)=\mathbb{E}_\pi[\sum_{t=h}^H r_{t}|s_h=s] ,
Q^\pi_h(s,a)=\mathbb{E}_\pi[\sum_{t=h}^H  r_{t}|s_h,a_h=s,a],\;\;\forall s,a\in\mathcal{S},\mathcal{A}.
$
Then Bellman (optimality) equation follows $\forall\, h\in[H]$:
\begin{align*}
&Q^\pi_h(s,a)=r_{h}(s,a)+P_{h}(\cdot|s,a)V^\pi_{h+1},\;\;V^\pi_h=\mathbb{E}_{a\sim\pi_h}[Q^\pi_h]\\
&Q^\star_h(s,a)=r_{h}(s,a)+P_{h}(\cdot|s,a)V^\star_{h+1},\; V^\star_h=\max_a Q^\star_h(\cdot,a)
\end{align*}
In this work, we will consider different MDPs with respective transition kernels and reward functions. We define the value function for policy $\pi$ under MDP $(\widetilde{r},\widetilde{P})$ as below
$$V^{\pi}(\widetilde{r},\widetilde{P})=\mathbb{E}_{\pi}[\sum_{h=1}^{H}\widetilde{r}_{h}|\widetilde{P}].$$
Also, the notation $\mathbb{P}_{\pi}[\cdot |\widetilde{P}]$ means the conditional probability under policy $\pi$ and MDP $\widetilde{P}$, the notation $\mathbb{E}_{\pi}[\cdot |\widetilde{P}]$ means the conditional expectation under policy $\pi$ and MDP $\widetilde{P}$. 

\noindent\textbf{Regret.} We measure the performance of online reinforcement learning algorithms by the regret. The regret of an algorithm is defined as
$$\text{Regret}(K) := \sum_{k=1}^{K}[V_{1}^\star(s_{1})-V_{1}^{\pi_{k}}(s_{1})],$$
where $\pi_{k}$ is the policy it employs at episode $k$. Let $K$ be the number of episodes that the agent plan to play and total number of steps is $T := KH$.

\noindent\textbf{Switching cost.} We adopt the global switching cost \citep{bai2019provably}, which simply measures how many times the algorithm changes its policy:
$$N_{switch} := \sum_{k=1}^{K-1} \mathds{1}\{\pi_{k}\neq \pi_{k+1}\} .$$
Global switching costs are more natural than local switching costs \footnote{$N_{switch}^{local}=\sum_{k=1}^{K-1}|\{(h,s)\in[H]\times\mathcal{S}: \pi_{k}^{h}(s)\neq\pi_{k+1}^{h}(s)\}|$} as they measure the number of times a deployed policy (which could then run asynchronously in a distributed fashion for an extended period of time) can be changed. 
\citet{bai2019provably}’s bound on local switching cost is thus viewed by them as a conservative surrogate of the global counterpart.  Similar to \citet{bai2019provably}, our algorithm also uses deterministic policies only. 

\section{Algorithms and Explanation} \label{sec:discussion}


\begin{figure*}[t]
	\centering     
	\includegraphics[width=140mm]{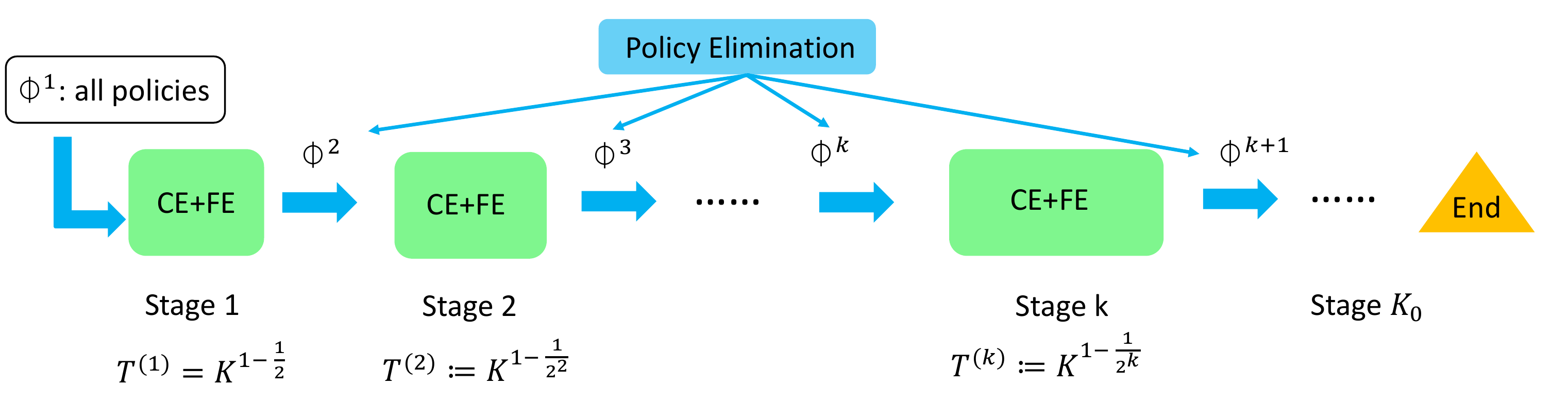}
	\caption{A visualization to explicate the procedures of APEVE (Algorithm~\ref{algo3}). In particular, the policy elimination procedures are conducted stage by stage, with increasing size $2T^{(k)}=2K^{1-\frac{1}{2^k}}$ at each stage. The Crude Exploration (CE) and Fine Exploration (FE) procedures within the stage apply Algorithm~\ref{algo1} and Algorithm~\ref{algo2}. APEVE (Algorithm~\ref{algo3}) is conducted to maintain a decreasing policy set $\phi^{k}$. The number of episodes satisfies $2\sum_{k=1}^{K_0}T^{(k)}=K$.}
	\label{fig:main}
\end{figure*}

\begin{algorithm}[tbh]
	\caption{Adaptive Policy Elimination by Value Estimation (APEVE)}\label{algo3}
	\begin{algorithmic}[1]
		\STATE \textbf{Require}: Number of episodes for exploration $K$, $r$ is the known deterministic reward. Universal constant $C$. Failure probability $\delta$.
		\STATE \textbf{Initialize}: $T^{(k)}=K^{1-\frac{1}{2^{k}}}$, $k\leq K_{0}=O(\log\log K)$, $\phi^{1}:=\{ \text{the set of all the deterministic policies}\}$, $\iota=\log(2HAK/\delta)$. 
		\FOR{$k=1,2,\cdots,K_{0}$}  
		\STATE \maroon{$\diamond$ Number of episodes in $k$-th stage:}
		\IF{$2(\sum_{i=1}^{k}T^{(i)})\geq K$} 
		\STATE $T^{(k)}=\frac{K-2(\sum_{i=1}^{k-1}T^{(i)})}{2}$. (o.w. $T^{(k)}=K^{1-\frac{1}{2^{k}}}$)
		\ENDIF
		\STATE \maroon{$\diamond$ Crude exploration using Algorithm~\ref{algo1}:} 
		\STATE $\mathcal{F}^{k}$,$P^{int,k}$ = Crude  Exploration$(\phi^{k},T^{(k)}).$
		\STATE \maroon{$\diamond$ Estimating $\widehat{P}^k$ using Algorithm~\ref{algo2}:}
		\STATE $\widehat{P}^{k}$ = Fine Exploration$(\mathcal{F}^{k},P^{int,k},T^{(k)},\phi^{k}).$
		\STATE \maroon{$\diamond$ Adaptive policy elimination from $\phi^k$:}
		\STATE $U^k=\emptyset$ 
		\FOR{$\pi\in\phi^{k}$}
		\IF{$V^{\pi}(r,\widehat{P}^{k})\leq sup_{\widehat{\pi}\in\phi^{k}}V^{\widehat{\pi}}(r,\widehat{P}^{k})-2C(\sqrt{\frac{H^{5}S^{2}A\iota}{T^{(k)}}}+\frac{S^{3}A^{2}H^{5}\iota}{T^{(k)}})$}
		\STATE Update $U^{k}\leftarrow U^k\cup \{\pi\}$.
		\ENDIF
		\ENDFOR
		\STATE $\phi^{k+1}\leftarrow\phi^{k}\backslash U^k$.
		\ENDFOR
	\end{algorithmic}
\end{algorithm}

\begin{algorithm}[tbh]
	\caption{Crude Exploration (for constructing infrequent tuples $\mathcal{F}$ and a reference transition $P^{int}$)}\label{algo1}
	\begin{algorithmic}[1]
		\STATE \textbf{Input}: Policy set $\phi$. Number of episodes $T$.
		\STATE \textbf{Initialize}: $T_{0}=\frac{T}{HSA}$, $C_{1}=6$, $\mathcal{F}=\emptyset$, $\mathcal{D}=\emptyset$, $\iota=\log(2HAK/\delta)$. $1_{h,s,a}$ is a reward function $r^{\prime}$ where $r^{\prime}_{h^{\prime}}(s^{\prime},a^{\prime})=\mathds{1}[(h^{\prime},s^{\prime},a^{\prime})=(h,s,a)]$. $s^{\dagger}$ is an additional absorbing state. $P^{int}$ is a transition kernel over the extended space $\mathcal{S}\cup\{s^\dagger\}\times\mathcal{A}$, initialized \emph{arbitrarily}.
		\STATE \textbf{Output}: Infrequent tuples $\mathcal{F}$. Intermediate transition kernel $P^{int}$.
		\FOR{$h=1,2,\cdots,H$}
		\STATE \maroon{$\diamond$ Construct and run policies to visit each state-action:}
		\FOR{$(s,a)\in \mathcal{S}\times \mathcal{A}$}
		\STATE $\pi_{h,s,a}=\mathrm{argmax}_{\pi\in \phi}V^{\pi}(1_{h,s,a},P^{int}).$
		\STATE Run $\pi_{h,s,a}$ for $T_{0}$ episodes, and add the trajectories into data set $\mathcal{D}$. 
		\ENDFOR
		\FOR{$(s,a,s')\in\mathcal{S}\times\mathcal{A}\times\mathcal{S}$}
		\STATE $N_{h}(s,a,s^{\prime})=\text{count of}\ (h,s,a,s^{\prime})$ in $\mathcal{D}$.
		\ENDFOR
		\STATE \maroon{$\diamond$ Use $\mathcal{F}$ to store the infrequent tuples:}
		\STATE $\mathcal{F}=\mathcal{F}\cup \{(h,s,a,s^{\prime})|N_{h}(s,a,s^{\prime})\leq C_{1}H^{2}\iota\}.$
        \STATE \maroon{$\diamond$ Update the intermediate transition kernel using Algorithm~\ref{algo_transition_kernel}:}
		\STATE $P^{int} = \text{EstimateTransition}(\mathcal{D}, \mathcal{F}, s^{\dagger},h,P^{int})$
		\STATE \maroon{$\diamond$ Clear the data set:}
 		\STATE Reset data set $\mathcal{D}=\emptyset$.
		\ENDFOR
		\STATE \textbf{Return}: $\{\mathcal{F},P^{int}\}$.
	\end{algorithmic}
\end{algorithm}

\begin{algorithm}[tbh]
	\caption{Fine Exploration (further exploration for accurate transition estimation)}\label{algo2}
	\begin{algorithmic}[1]
		\STATE \textbf{Input}: Infrequent tuples $\mathcal{F}$. Intermediate transition kernel $P^{int}$. Number of episodes $T$. Policy set $\phi$.
		\STATE \textbf{Initialize}: $T_{0}=\frac{T}{HSA}$, $\mathcal{D}=\emptyset$. $1_{h,s,a}$ is a reward function $r^{\prime}$ where $r^{\prime}_{h^{\prime}}(s^{\prime},a^{\prime})=\mathds{1}[(h^{\prime},s^{\prime},a^{\prime})=(h,s,a)]$. $\widehat{P}=P^{int}$.
		\STATE \textbf{Output}: Empirical estimate $\widehat{P}$.
		\STATE \maroon{$\diamond$ Construct and run policies to visit each state-action:}
		\FOR{$(h,s,a)\in[H]\times \mathcal{S}\times \mathcal{A}$}
		\STATE $\pi_{h,s,a}=\mathrm{argmax}_{\pi\in \phi}V^{\pi}(1_{h,s,a},P^{int}).$
		\STATE Run $\pi_{h,s,a}$ for $T_{0}$ episodes, and add the trajectories into data set $\mathcal{D}$.
		\ENDFOR
		\STATE \maroon{$\diamond$ Construct an empirical estimate for $\widetilde{P}$ using Algorithm~\ref{algo_transition_kernel}:}
		\FOR{$h\in [H]$}
		\STATE $\widehat{P} = \mathrm{EstimateTransition}(\mathcal{D}, \mathcal{F}, s^{\dagger},h,\widehat{P}).$
		\ENDFOR
		\STATE \textbf{Return} $\widehat{P}$.
	\end{algorithmic}
\end{algorithm}
Our algorithm generalizes the arm-elimination algorithm of \citet{cesa2013online} for bandits to a policy-elimination algorithm for RL. The high-level idea of our policy elimination algorithm is the following. We maintain a \emph{version space} $\phi$ of remaining policies and iteratively refine the estimated values of all policies in $\phi$ while using these values to eliminate those policies that are \emph{certifiably} suboptimal. The hope is that towards the end of the algorithms, all policies that are not eliminated are already nearly optimal.

As we explained earlier, the challenge is to estimate the value function of all $A^{HS}$ policies using $\mathrm{poly}(H,S,A)$ samples. 
This uniform convergence problem typically involves estimating the transition kernels, but it requires solving an exploration problem to even visit a particular state-action pair once.  In addition, some states cannot be visited frequently by any policy. To address these issues, we need to construct 
a surrogate MDP (known as an ``absorbing MDP'') with an absorbing state. This absorbing MDP replaces these troublesome states with an absorbing state $s^\dagger$, such that all remaining states can be visited sufficiently frequently by some policy in $\phi$. Moreover, its value function uniformly approximates the original MDP for all policies of interest. This reduces the problem to estimating the transition kernel of the absorbing MDP.

\noindent\textbf{Adaptive policy elimination.} 
The overall workflow of our algorithm --- Adaptive Policy Elimination by Value Estimation (APEVE) --- is given in Algorithm~\ref{algo3} and illustrated graphically in Figure~\ref{fig:main}. It first divides a budget of $K$ episodes into a sequence of stages with increasing length $2T^{(k)}:= 2K^{1-1/(2^k)}$ for $k=1,2,3,...$. By Lemma~\ref{lem50}, the total number of stages $K_{0}=O(\log\log K)$.
 
Each stage involves three steps. 
\begin{description}
\itemsep0em
\item[Step 1. Crude exploration] Explore each $h,s,a$ layer-by-layer \emph{from scratch} based on the current version-space $\phi$. Construct an absorbing MDP $\widetilde{P}$ and a \emph{crude} intermediate estimate ($P^{int}$) of $\widetilde{P}$.  
\item[Step 2. Fine exploration] Explore each $h,s,a$ with the crude estimate of the absorbing MDP. Construct a more refined estimate ($\widehat{P}$) of the absorbing MDP's $\widetilde{P}$. 
\item[Step 3. Policy elimination] Evaluate all policies in $\phi$ using $\widehat{P}$. Update the version-space $\phi$ by eliminating all policies whose value upper confidence bound (UCB) is smaller than the \emph{mode} (max over $\pi\in \phi$) of the lower confidence bound (LCB). 
\end{description}

As the algorithm proceeds, under the high probability event that our confidence bounds are valid, the optimal policy will not be eliminated. After each stage, the performance of all policies that remain will be better than the LCB of the optimal policy, which itself will get closer to the actual valuation function as we collect more data.

Next, we break down the key components in ``Crude Exploration'' and ``Fine Exploration'' and explain how they work.

\noindent\textbf{Layerwise ``Exploration'' in Algorithm~\ref{algo1}.} Our goal is to learn an accurate enough\footnote{$\frac{1}{H}$-multiplicatively accurate, details in Definition~\ref{def3}} estimate of $P_{h}(s^{\prime}|s,a)$ for any tuple $(h,s,a,s^{\prime})$ and it suffices to visit this tuple $O(H^{2}\iota)$ times.\footnote{Detailed proof in Lemma~\ref{lem6}} Therefore, we try to visit each tuple as much as possible using policies from the input policy set $\phi$. However, it is possible that some $(h,s,a,s^{\prime})$ tuples are hard to visit by any policy in the remaining policy set. To address this problem, we use a set $\mathcal{F}$ to store all the tuples that have not been visited enough times such that for the tuples not in $\mathcal{F}$, we can get an accurate enough estimate, while for the tuples in $\mathcal{F}$, we will prove that they have little influence on the value function. 

In the algorithm, we apply the trick of layerwise exploration. During the exploration of the $h$-th layer, we use the intermediate MDP $P^{int}$ to construct $\pi_{h,s,a}$ that can visit $(h,s,a)$ with the largest probability under $P^{int}$. Then we run each $\pi_{h,s,a}$ for the same number of episodes. Using the data set $\mathcal{D}$  we collect, the $h$-th layer of $P^{int}$ (i.e. $P^{int}_h$) is updated using Algorithm~\ref{algo_transition_kernel}. 

Given infrequent tuples $\mathcal{F}$, the absorbing MDP is constructed as in Definition~\ref{def2}. In the construction, we first let $\widetilde{P}=P$, then for $(h,s,a,s^{\prime})\in\mathcal{F}$, we move the probability of $\widetilde{P}_{h}(s^{\prime}|s,a)$ to $\widetilde{P}_{h}(s^{\dagger}|s,a)$.
\begin{definition}[The absorbing MDP $\widetilde{P}$]\label{def2}
    Given $\mathcal{F}$ and $P$, $\forall (h,s,a,s^{\prime}) \notin \mathcal{F}$, let
	$\widetilde{P}_{h}(s^{\prime}|s,a)=P_{h}(s^{\prime}|s,a).$
	For any $(h,s,a,s^{\prime})\in \mathcal{F}$, $\widetilde{P}_{h}(s^{\prime}|s,a)=0.$ For any $(h,s,a)\in[H]\times\mathcal{S}\times\mathcal{A}$, define $\widetilde{P}_{h}(s^{\dagger}|s^{\dagger},a)=1$ and \[ \widetilde{P}_{h}(s^{\dagger}|s,a)=1-\sum_{s^{\prime}\in\mathcal{S}:(h,s,a,s^{\prime})\notin \mathcal{F}}\widetilde{P}_{h}(s^{\prime}|s,a).\]
\end{definition}
\vspace{-1em}
According to the construction in Algorithm~\ref{algo_transition_kernel}, $P^{int}$ is the empirical estimate of $\widetilde{P}$. We will show that with high probability, for $(h,s,a,s^{\prime})\in [H]\times\mathcal{S}\times\mathcal{A}\times\mathcal{S}$, either  $(1-\frac{1}{H})P^{int}_{h}(s^{\prime}|s,a)\leq \widetilde{P}_{h}(s^{\prime}|s,a) \leq (1+\frac{1}{H})P^{int}_{h}(s^{\prime}|s,a)$  or $P^{int}_{h}(s^{\prime}|s,a)=\widetilde{P}_{h}(s^{\prime}|s,a)=0$. Based on this property, we can prove that $\pi_{h,s,a}$'s are efficient in exploration. 
Algorithm~\ref{algo_transition_kernel} and detailed explanation are deferred to Appendix~\ref{appb1}.

In Algorithm~\ref{algo1}, the reward function $1_{h,s,a}$ is defined under the original MDP while $P^{int}$ is a transition kernel of the absorbing MDP. In addition, $\pi_{h,s,a}$ is a policy under the absorbing MDP and we need to run it under the original MDP. The transition between the original MDP and the absorbing MDP is deferred to Appendix~\ref{appb2}.

\noindent\textbf{Fine exploration by Algorithm~\ref{algo2}.} The idea behind Algorithm~\ref{algo2} is that with high probability, we can use $P^{int}$ to construct policies to visit each tuple $(h,s,a)$ with the guarantee that $\sup_{a\in\mathcal{A},\pi\in\phi}\frac{V^{\pi}(1_{h,s,a},\widetilde{P})}{\mu_{h}(s,a)}\leq 12HSA$, where $\mu$ is the distribution of our data. This inequality is similar to the result of Theorem 3.3 in \citep{jin2020reward}, 
which means that we can get a similar result to Lemma 3.6 in \citep{jin2020reward} that $V^{\pi}(r^{\prime},\widehat{P})$ is an accurate estimate of $V^{\pi}(r^{\prime},\widetilde{P})$ simultaneously for all $\pi\in\phi$ and any reward function $r^{\prime}$. 

For each $(h,s,a)$, the algorithm finds the policy $\pi_{h,s,a}$ from $\phi$ that visits $(h,s,a)$ with the largest probability under $P^{int}$. Then each $\pi_{h,s,a}$ is run for the same number of episodes over all $(h,s,a)$. At last, $\widehat{P}$ is calculated as an empirical estimate of $\widetilde{P}$ by using Algorithm~\ref{algo_transition_kernel} and the data set $\mathcal{D}$.

\section{Main Results of APEVE} \label{sec:result}

In this section, we will state our main results, which formalizes the algorithmic ideas we explained in the previous section. 
\begin{theorem}[Regret and switching cost of Algorithm~\ref{algo3}]\label{the1}
	With probability $1-\delta$, Algorithm~\ref{algo3} will have regret bounded by $O(\sqrt{H^{5}S^{2}AK\cdot \log(\frac{2}{\delta}HAK) }\cdot \log\log K+S^{3}A^{2}H^{5}K^{\frac{1}{4}}\cdot \log(\frac{2}{\delta}HAK))=\widetilde{O}(\sqrt{H^{4}S^{2}AT})$. Furthermore, the global switching cost of Algorithm~\ref{algo3} is $O(HSA\log\log T)$ while the timestep for policy switching can be decided before the algorithm starts.
\end{theorem}
Recall that the number of episodes $K = T/H$ where $T$ is the number of steps. This theorem says that Algorithm~\ref{algo3} obtains a regret bound that is optimal in $T$ while changing (deterministic) policies for only $O(HSA\log\log T)$ times.

The proof of Theorem~\ref{the1} is sketched in Section~\ref{sec:info_lower} with pointers to more detailed arguments to the full proof in the appendix. Now we discuss a few interesting aspects of the result.

\noindent\textbf{Near optimal switching cost.} Our algorithm achieves a switching cost that improves over existing work with $\sqrt{T}$ regret. We also prove the following information-theoretic limit which says that the global switching cost of APEVE (Algorithm \ref{algo3}) is optimal up to constant.

\begin{theorem}[Lower bound for global switching cost under optimal regret bound]\label{thm:lower_bound}
	If $S\leq A^{\frac{H}{2}}$, for any algorithm with near-optimal $\widetilde{O}^*(\sqrt{T})$ regret bound, the global switching cost is at least $\Omega(HSA\log\log T)$.
\end{theorem}

As a byproduct, our proof technique naturally leads to the following lower bound for global switching cost for any algorithm with no regret.

\begin{theorem}[Lower bound for global switching cost under sub-linear regret bound]\label{the2}
	If $S\leq A^{\frac{H}{2}}$, for any algorithm with sub-linear regret bound, the global switching cost is at least $\Omega(HSA)$.
\end{theorem}
Both proofs are deferred to Appendix~\ref{appd}.  Theorem~\ref{the2} is a stronger conclusion than the existing $\Omega(HSA)$ lower bound for \emph{local} switching cost \citep{bai2019provably} because global switching cost is smaller than local switching cost. An $\Omega(HSA)$ lower bound on local switching cost can only imply an $\Omega(A)$ lower bound on global switching cost.


\noindent\textbf{Near-optimal adaptivity.} Interestingly, our algorithm also enjoys low-adaptivity besides low-switching cost in the batched RL setting \citep{perchet2016batched,gao2019batched}, because the length of each batch can be determined ahead of time and we do not require monitoring within each batch. APEVE (Algorithm~\ref{algo3}) runs with $O(H\log\log T)$ batches.  In Appendix~\ref{appi}, we present APEVE+ (Algorithm~\ref{algo6}), which further improves the batch complexity to $O(H+\log\log T)$ while maintaining the same regret $\widetilde{O}(\sqrt{H^{4}S^{2}AT})$. These results nearly matches the existing lower bound $\Omega(\frac{H}{\log T}+\log\log T)$ due to Theorem B.3 in \citep{huang2022towards} (for the $\frac{H}{\log T}$ term) and Corollary 2 in \citep{gao2019batched} (for the $\log\log T$ term).


\noindent\textbf{Dependence on $H,S,A$ in the regret.}  As we explained earlier our regret bound is optimal in $T$. However, there is a gap of  $\sqrt{H^2S}$ when compared to the information-theoretic limit of $\Omega(\sqrt{H^2SAT})$ that covers all algorithms (including those without switching cost constraints).  We believe our analysis is tight and further improvements on $H,S$ will require new algorithmic ideas.  It is an intriguing open problem whether any algorithm with $\log\log T$ switching cost need to have $\Omega(S^2)$ dependence.

\textbf{Computational efficiency.} One weakness of our APEVE (Algorithm \ref{algo3}) is that it is not computationally efficient. APEVE needs to explicitly go over each element of the \emph{version spaces} --- the sets of remaining policies --- to implement policy elimination. It remains an interesting open problem to design a polynomial-time algorithm for RL with optimal switching cost. A promising direction to achieve computational efficiency is to avoid explicitly representing the version spaces, or to reduce to ``optimization oracles''.
We leave a full exploration of these ideas to a future work.

\section{Low Adaptive Reward-Free Exploration}


\begin{algorithm}
	\caption{Low Adaptive Reward-Free Exploration (LARFE)}\label{algo4}
	\begin{algorithmic}[1]
		\STATE \textbf{Input}: Episodes for crude exploration $N_{0}$, episodes for fine exploration $N$. Failure probability $\delta$.
		\STATE \textbf{Initialize}: $\phi^{1}= \{\text{the set of all deterministic policies}\}$, $\iota=\log(2HA(N_{0}+N)/\delta)$. 
		\STATE \textbf{Output}: $\widehat{\pi}^{r}$ for any reward function $r$.
		\STATE \maroon{$\diamond$ Crude exploration using  Algorithm~\ref{algo1}:}
		\STATE $\mathcal{F}$,$P^{int}$ = $\text{Crude Exploration}(\phi^{1},N_{0})$.
		\STATE \maroon{$\diamond$ Estimate $\widehat{P}$ using  Algorithm~\ref{algo2}:}
		\STATE $\widehat{P}$ = $\text{Fine Exploration}(\mathcal{F},P^{int},N,\phi^{1}).$
		\STATE \maroon{$\diamond$ For any reward function, output the optimal policy under empirical MDP by value iteration:}
		\STATE $\widehat{\pi}^{r}=\mathrm{argmax}_{\pi\in\phi^{1}}V^{\pi}(r,\widehat{P})$ for any $r$.
		\STATE \textbf{Return} $\{\widehat{\pi}^{r}\}$.
	\end{algorithmic}
\end{algorithm}

In this section, we further consider the new setting of \emph{low adaptive reward-free exploration}. Specifically, due to its nature that Crude Exploration (Algorithm~\ref{algo1}) and Fine Exploration (Algorithm~\ref{algo2}) do not use any information about the reward function $r$, these two algorithms can be leveraged in reward-free setting. LARFE (Algorithm~\ref{algo4}) is an algorithm that tackles reward-free exploration while maintaining the low switching cost at the same time.

In LARFE, we use Crude Exploration (Algorithm~\ref{algo1}) to construct the infrequent tuples $\mathcal{F}$ and the intermediate MDP $P^{int}$. Then the algorithm uses Fine Exploration (Algorithm~\ref{algo2}) to get an empirical estimate $\widehat{P}$ of the absorbing MDP $\widetilde{P}$. At last, for any reward function $r$, the algorithm outputs the optimal policy under the empirical MDP, which can be done efficiently by value iteration.

Theorem~\ref{the3} provides the switching cost and sample complexity of Algorithm~\ref{algo4} (whose proof is deferred to Appendix~\ref{appe}).
\begin{theorem}\label{the3}
	The global switching cost of Algorithm~\ref{algo4} is bounded by $2HSA$. There exists a constant $c>0$ such that, for any $\epsilon >0$ and any $\delta>0$, if the number of total episodes $K$ satisfies that $$K>c\cdot (\frac{H^{5}S^{2}A\cdot \iota^{\prime}}{\epsilon^{2}}+\frac{S^{3}AH^{5}\cdot \iota^{\prime}}{\epsilon}),$$ where $\iota^{\prime}=\log(\frac{HSA}{\epsilon\delta})$, then there exists a choice of $N_{0}$ and $N$ such that $N_{0}+N=K$ and with probability $1-\delta$, for any reward function $r$, Algorithm~\ref{algo4} will output a policy $\widehat{\pi}^{r}$ that is $\epsilon$-optimal.
\end{theorem}

\textbf{Take-away of Theorem~\ref{the3}.} First of all, one key feature of LARFE is that the global switching cost is always bounded by $O(HSA)$ and this holds true for any $K$ (\emph{i.e.} independent of the PAC guarantee). Second, as a comparison to \citet{jin2020reward} regarding reward-free exploration, their episodes needed is $O(\frac{H^{5}S^{2}A\iota^{\prime}}{\epsilon^{2}}+\frac{S^{4}AH^{7}(\iota^{\prime})^{3}}{\epsilon})$. Our result matches this in the main term and does better in the lower order term. In addition, our algorithm achieves near optimal switching cost while the use of EULER in \citet{jin2020reward} can have switching cost equal to the number of episodes $N_{0}$. This means that Crude Exploration (Algorithm~\ref{algo1}) is efficient in the sense of sample complexity and switching cost when doing exploration. 

\textbf{Explore-First with LARFE.} Given the number of episodes $K$ and the corresponding number of steps $T=KH$, we can apply LARFE (Algorithm \ref{algo4}) for the first $K_0$ episodes, then run the greedy policy $\widehat{\pi}^{r}$ returned by LARFE (with $r$ to be the real reward) for the remaining episodes. Then the regret is bounded by $HK_0+\widetilde{O}(K\cdot\sqrt{\frac{H^5S^2A}{K_0}})$ with high probability. By selecting $K_0=K^{\frac{2}{3}}HS^{\frac{2}{3}}A^{\frac{1}{3}}$, the regret can be bounded by $\widetilde{O}(K^{\frac{2}{3}}H^2S^{\frac{2}{3}}A^{\frac{1}{3}})=\widetilde{O}(T^{\frac{2}{3}}H^{\frac{4}{3}}S^{\frac{2}{3}}A^{\frac{1}{3}})$, as shown in Table \ref{tab:comparison}. We highlight that Explore-First w. LARFE matches the lower bound given by Theorem \ref{the2}.

\textbf{Optimal switching cost in Pure Exploration.}
 Since any best policy identification (i.e., Pure Exploration) algorithm with polynomial sample complexity  can be used to construct a no-regret learning algorithm with an Explore-First strategy, Theorem \ref{the2} implies that $\Omega(HSA)$ is a switching cost lower bound for the pure exploration problem too, thus also covering the task-agnostic / reward-free extensions. LARFE implies that one can achieve nearly optimal sample complexity ($\widetilde{O}^*(1/\epsilon^2)$) while achieving the best possible switching cost of $O(HSA)$.  
 
 \textbf{Separation of Regret Minimization and Pure Exploration in RL.} Note that achieving a near-optimal $\widetilde{O}^*(\sqrt{T})$ regret requires an additional factor of $\log\log T$ in the switching cost (Theorem~\ref{thm:lower_bound}). This provides an interesting separation of the hardness between low-adaptive regret minimization and low-adaptive pure exploration in RL.


\section{Proof Overview}\label{sec:info_lower}

Due to the space constraint, we could only sketch the proof of Theorem~\ref{the1} as the $\log\log T$ switching cost is our major contribution. The analysis involves two main parts: the switching cost bound and the regret bound. The switching cost bound directly results from the schedule of Algorithm~\ref{algo3}.

\noindent\textbf{Upper bound for switching cost.}
First of all, we have the conclusion that the global switching cost of Algorithm~\ref{algo3} is bounded by $O(HSA\log\log T)$. This is because the global switching cost of both Algorithm~\ref{algo1} and Algorithm~\ref{algo2} are bounded by $HSA$ and the fact that the number of stages satisfy $K_{0}=O(\log\log T)$.

However, such an elimination schedule requires the algorithm to run the same deterministic policy for a long period of time before being able to switch to another policy, which is the main technical challenge to the regret analysis.

\noindent\textbf{Regret analysis.} At the heart of the regret analysis is to construct a uniform off-policy evaluation bound that covers all remaining deterministic policies.
The remaining policy set at the beginning of stage $k$ is $\phi^{k}$. Assume we can estimate all $V^{\pi}(r,P)$ ($\pi\in\phi^{k}$) to $\epsilon_{k}$ accuracy with high probability, then we can eliminate all policies that are at least $2\epsilon_{k}$ sub-optimal in the sense of estimated value function. Therefore, the optimal policy will not be eliminated and all the policies remaining will be at most $4\epsilon_{k}$ sub-optimal with high probability. Summing up the regret of all stages, we have with high probability, the total regret is bounded by
\vspace{-1em}
\begin{equation}\label{equr}
\text{Regret}(K)\leq 2HT^{(1)}+\sum_{k=2}^{K_{0}}2T^{(k)}\times 4\epsilon_{k-1}.
\end{equation}
The following lemma gives an bound of $\epsilon_{k-1}$ using the model-based plug-in estimator with our estimate $\widehat{P}$ of the absorbing MDP.
\begin{lemma}\label{lem16}
	There exists a constant $C$, such that with probability $1-\delta$, it holds that for any $k$ and $\pi\in\phi^{k}$,
	$$|V^{\pi}(r,\widehat{P}^{k})-V^{\pi}(r,P)|\leq C(\sqrt{\frac{H^{5}S^{2}A\iota}{T^{(k)}}}+\frac{S^{3}A^{2}H^{5}\iota}{T^{(k)}}).$$
\end{lemma}
The proof of Lemma~\ref{lem16} involves controlling both the ``bias'' and ``variance'' part of the estimate. The ``bias'' refers to the difference between the true MDP and the absorbing MDP, and the ``variance'' refers to the statistical error in estimating the surrogate value functions of the absorbing MDP using our estimate 
$\widehat{P}^{k}$.

From the proof of \citep{jin2020reward}, we know that if we can visit each $(h,s,a)$ frequently enough, which means the visitation probability is maximal up to a constant factor, then the empirical transition kernel is enough for a uniform approximation to $V^{\pi}(r,P)$. The absorbing MDP $\widetilde{P}$ is the key to guarantee the condition of frequent visitation.

For the ease of illustration, in the following discussion, we omit the stage number $k$ and the discussion holds true for all $k$. Besides, in all of the following lemmas in this section, ``with high probability'' means with probability at least $1-\delta$ and $\iota = \log(2HAK/\delta)$.

\subsection{The ``bias'': difference between $P$ and $\widetilde{P}$}
To analyze the difference between the true MDP $P$ and the absorbing MDP $\widetilde{P}$, we first sketch some properties of the intermediate transition kernel $P^{int}$.

\noindent\textbf{Accuracy of $P^{int}$.}
It holds that if the visitation number of a tuple $(h,s,a,s^{\prime})$ is larger than $O(H^{2}\iota)$, with high probability\footnote{Proof using empirical Bernstein's inequality in Lemma~\ref{lem6}}, 
{\small
\begin{align}\label{equ0}(1-\frac{1}{H})P^{int}_{h}(s^{\prime}|s,a)\leq \widetilde{P}_{h}(s^{\prime}|s,a) \leq (1+\frac{1}{H})P^{int}_{h}(s^{\prime}|s,a).
\end{align}
}According to the definition of $\mathcal{F}$ in Algorithm~\ref{algo1} and the construction of $\widetilde{P}$, $P^{int}$, we have Equation~\eqref{equ0} is true for any $(h,s,a,s^{\prime})\in [H]\times\mathcal{S}\times\mathcal{A}\times\mathcal{S}$. Then we have\footnote{Proof using multiplicative bound in Lemma~\ref{lem2}} for any $(h,s,a)\in [H]\times\mathcal{S}\times\mathcal{A}$, $\pi\in\phi$, $$\frac{1}{4}V^{\pi}(1_{h,s,a},P^{int}) \leq V^{\pi}(1_{h,s,a},\widetilde{P}) \leq 3V^{\pi}(1_{h,s,a},P^{int}).$$ Because $\pi_{h,s,a}=\mathrm{argmax}_{\pi\in \phi}V^{\pi}(1_{h,s,a},P^{int})$,
\begin{equation}\label{equc} V^{\pi_{h,s,a}}(1_{h,s,a},\widetilde{P})
\geq\frac{1}{12}\sup_{\pi\in \phi}V^{\pi}(1_{h,s,a},\widetilde{P}),
\end{equation}
which shows that $\pi_{h,s,a}$ is efficient in visiting the tuple $(h,s,a)$.

\noindent\textbf{Uniform bound on $|V^{\pi}(r^{\prime},P) - V^{\pi}(r^{\prime},\widetilde{P})|$.}
Now we are ready to bound $\sup_{\pi\in\phi}\sup_{r^{\prime}}|V^{\pi}(r^{\prime},P)-V^{\pi}(r^{\prime},\widetilde{P})|$ by bounding $\sup_{\pi\in\phi}\mathbb{P}_{\pi}[\mathcal{B}]$, where the bad event $\mathcal{B}$ is defined as\footnote{The detailed definition can be found in Definition~\ref{def4}} the event where a trajectory visits some tuple in $\mathcal{F}$. Then we have the key lemma showing that the infrequent tuples are hard to visit by any policy in $\phi$. 

\begin{lemma}\label{lem9}
	With high probability, $\sup_{\pi\in\phi} \mathbb{P}_{\pi}[\mathcal{B}]\leq O(\frac{S^{3}A^{2}H^{4}\iota}{T})$.
\end{lemma}

With Lemma~\ref{lem9}, we are able to bound the difference between $P$ and $\widetilde{P}$ in the sense of value function. 

\begin{lemma}\label{lem10}
	With high probability, it holds that $$0\leq V^{\pi}(r^{\prime},P)-V^{\pi}(r^{\prime},\widetilde{P})\leq O(\frac{S^{3}A^{2}H^{5}\iota}{T}),$$ for any policy $\pi\in\phi$ and reward function $r^{\prime}$.
\end{lemma}

Therefore, the ``bias'' term $\sup_{\pi\in\phi}|V^{\pi}(r,P)-V^{\pi}(r,\widetilde{P})|$ can be bounded by the right hand side of Lemma~\ref{lem10} as a special case.


\subsection{The ``variance'': difference between $\widetilde{P}$ and $\widehat{P}$}
Because of the fact that with high probability, Equation~\eqref{equc} holds, we have the following key lemma.
\begin{lemma}\label{lem14}
	With high probability, for any policy $\pi\in\phi$ and any reward function $r^{\prime}$, $$|V^{\pi}(r^{\prime},\widehat{P})-V^{\pi}(r^{\prime},\widetilde{P})|=O(\sqrt{\frac{H^{5}S^{2}A\iota}{T}}).$$
\end{lemma}

Therefore, the ``variance'' term $\sup_{\pi\in\phi}|V^{\pi}(r,\widehat{P})-V^{\pi}(r,\widetilde{P})|$ can be bounded by the right hand side of Lemma~\ref{lem14} as a special case.

\subsection{Put everything together}
Combining the bounds of the ``bias'' term and the ``variance'' term, because of triangular inequality, we have the conclusion in Lemma~\ref{lem16} holds. Then the proof of regret bound is completed by plugging in $\epsilon_{k}=O(\sqrt{\frac{H^{5}S^{2}A\iota}{T^{(k)}}}+\frac{S^{3}A^{2}H^{5}\iota}{T^{(k)}})$ in equation~\eqref{equr}.

\section{Conclusion and Future Works}\label{sec:conclusion}

This work studies the well-motivated \emph{low switching online reinforcement learning} problem. Under the non-stationary tabular RL setting, we design the algorithm \emph{Adaptive Policy Elimination by Value Estimation} (APEVE) which achieves $\widetilde{O}(\sqrt{H^{4}S^{2}AT})$ regret while switching its policy for at most $O(HSA\log\log T)$ times. Under the reward-free exploration setting, we design the \emph{Low Adaptive Reward-Free Exploration} (LARFE), which achieves $\widetilde{O}(\frac{H^{5}S^{2}A\iota}{\epsilon^{2}})$ sample complexity with switching cost at most $2HSA$. We also prove lower bounds showing that these switching costs are information-theoretically optimal among algorithms that achieve nearly optimal regret or sample complexity. These results nicely settled the open problem on the optimal low-switching RL raised by \citet{bai2019provably} (and revisited by \citet{zhang2020almost,gao2021provably}) for the tabular setting. 

It remains open to address computational efficiency, characterize the optimal dependence on $H,S,A$ in the regret bound, study RL with function approximation, as well as to make the the algorithm practical. We leave those as future works and invite the broader RL research community to join us in the quest. Ideas and techniques developed in this paper could be of independent interest in other problems.




\section*{Acknowledgments}
The research is partially supported by NSF Awards \#2007117 and \#2003257. The authors would like to thank Yichen Feng and Mengye Liu for helpful discussion at an early stage of this project, as well as Tianchen Yu and Chi Jin for clarifying the proof of Lemma C.2 in \citet{jin2020reward} (Lemma~\ref{lem13} in this paper). DQ would like to thank Fuheng Zhao for some helpful suggestions on writing.

\bibliographystyle{plainnat}
\bibliography{sections/stat_rl}

\newpage
\appendix
\onecolumn
\begin{onecolumn}
\section{Extended related work}\label{appr}
\textbf{Low regret reinforcement learning algorithms}
There has been a long line of works \citep{brafman2002r,kearns2002near,jaksch2010near,osband2013more,agrawal2017posterior,jin2018q} focusing on regret minimization for online reinforcement learning. \citet{azar2017minimax} used model-based algorithm (UCB-Q-values) to achieve the optimal regret bound $\widetilde{O}(\sqrt{HSAT})$ for stationary tabular MDP. \citet{dann2019policy} used algorithm ORLC to match the lower bound of regret and give policy certificates at the same time. \citet{zhang2020almost} used Q-learning type algorithm (UCB-advantage) to achieve the optimal $\widetilde{O}(\sqrt{H^{2}SAT})$ regret for non-stationary tabular MDP. \citet{zanette2019tighter} designed the algorithm EULER to get a problem dependent regret bound, which also matches the lower bound.

\par
\textbf{Reward-free exploration}
 \citet{jin2020reward} first studied the problem of reward-free exploration, they used a regret minimization algorithm EULER \citep{zanette2019tighter} to visit each state as much as possible. The sample complexity for their algorithm is $\widetilde{O}(H^{5}S^{2}A/\epsilon^{2})$ episodes. \citet{kaufmann2021adaptive} designed an algorithm RF-UCRL by building upper confidence bound for any reward function and any policy, their algorithm needs of order $\widetilde{O}((S^{2}AH^{4}/\epsilon^{2})$ episodes to output a near-optimal policy for any reward function with high probability. \citet{menard2021fast} constructed a novel exploration bonus of order $\frac{1}{n}$ and their algorithm achieved sample complexity of $\widetilde{O}((S^{2}AH^{3}/\epsilon^{2})$. \citet{zhang2020nearly} considered a more general setting with stationary transition kernel and uniformly bounded reward. They designed a novel condition to achieve the optimal sample complexity $\widetilde{O}((S^{2}A/\epsilon^{2})$ under their setting. Also, their result can be used to achieve $\widetilde{O}((S^{2}AH^{2}/\epsilon^{2})$ sample complexity under traditional setting where $r_{h}\in [0,1]$, this result matches the lower bound. \citet{wang2020reward} and \citet{zanette2020provably} analyzed reward-free exploration under the setting of linear MDP. There is a similar setting named task-agnostic exploration. \citet{zhang2020task} designed an algorithm: UCB-Zero that finds $\epsilon$-optimal policies for $N$ arbitrary tasks after at most $\widetilde{O} (H^{5}SA\log N/\epsilon^{2})$ exploration episodes. 
A concurrent work \citep{huang2022towards} analyzed low adaptive reward-free exploration under linear MDP. In our work, we consider low adaptive reward-free exploration under tabular MDP, our switching cost is of the same order as \citep{huang2022towards} and our sample complexity is much smaller than theirs if directly plugging in $d=SA$ in their bounds.

\paragraph{Bandit algorithms with limited adaptivity}
There has been a long history of works about multi-armed bandit algorithms with low adaptivity \citep{cesa2013online,perchet2016batched,gao2019batched,esfandiari2021regret}. \citet{cesa2013online} designed an algorithm with $\widetilde{O}(\sqrt{KT})$ regret using $O(\log\log T)$ batches. \citet{perchet2016batched} proved a regret lower bound of  $\Omega(T^{\frac{1}{1-2^{1-M}}})$ for algorithms within $M$ batches under $2$-armed bandit setting, which means $\Omega(\log\log T)$ batches are necessary for a regret bound of $\widetilde{O}(\sqrt{T})$. The result is generalized to $K$-armed bandit by \citet{gao2019batched}. We will show the connection and difference between this setting and the low switching setting. In batched bandit problems, the agent decides a sequence of arms and observes the reward of each arm after all arms in that sequence are pulled. More formally, at the beginning of each batch, the agent decides a list of arms to be pulled. Afterwards, a list of (arm,reward) pairs is given to the agent. Then
the agent decides about the next batch. The batch sizes could be chosen non-adaptively or adaptively. In a non-adaptive algorithm, the batch sizes should be decided before the algorithm starts, while in an adaptive algorithm,
the batch sizes may depend on the previous observations. \citep{esfandiari2021regret}. Under the switching cost setting, the algorithm can monitor the data stream and decide to change policy at any time, which means an algorithm with low switching cost can have $\Omega(T)$ batches. In addition, algorithms with limited batches can have large switching cost because in one batch, the algorithm can use different policies. Under batched bandit problem, algorithms with at most $M$ batches can have a $MK$ upper bound for switching cost. However, if we generalize batched bandit to batched RL, algorithms with at most $M$ batches can have $A^{SH}M$ switching cost in the worst case. We conclude that an upper bound of batches and an upper bound of switching cost can not imply each other in the worst case.

\section{Missing algorithm: EstimateTransition (Algorithm~\ref{algo_transition_kernel}) and some explanation}\label{appb1}
\begin{algorithm}
	\caption{Compute Transition Kernel (EstimateTransition)}\label{algo_transition_kernel}
	\begin{algorithmic}[1]
		\STATE \textbf{Require}: Data set $\mathcal{D}$, infrequent tuples $\mathcal{F}$, absorbing state $s^{\dagger}$, the target layer $h$, transition kernel $P$.
		\STATE \textbf{Output}: Estimated transition kernel $P$ from data set $\mathcal{D}$.
		\STATE \maroon{$\diamond$ Count the visitation number of each state-action pairs from the target layer $h$:}
		\FOR{$(s,a,s')\in\mathcal{S}\times\mathcal{A}\times\mathcal{S}$}
		\STATE $N_{h}(s,a,s^{\prime})=\text{count of}\ (h,s,a,s^{\prime})$ in $\mathcal{D}$.
		\STATE $N_{h}(s,a)=\text{count of}\ (h,s,a)$ in $\mathcal{D}$.
		\ENDFOR
        \STATE \maroon{$\diamond$ Update the $h$-th layer of the transition kernel:}
		\FOR{$(s,a,s^{\prime})\in \mathcal{S}\times \mathcal{A}\times\mathcal{S}$ s.t. $(h,s,a,s^{\prime})\in \mathcal{F}$}
		\STATE $P_{h}(s^{\prime}|s,a)=0.$
		\ENDFOR
		\FOR{$(s,a,s^{\prime})\in \mathcal{S}\times \mathcal{A}\times\mathcal{S}$ s.t. $(h,s,a,s^{\prime})\notin \mathcal{F}$}
		\STATE $P_{h}(s^{\prime}|s,a)=\frac{N_{h}(s,a,s^{\prime})}{N_{h}(s,a)}.$ 
		\ENDFOR
		\FOR{$(s,a)\in \mathcal{S}\times \mathcal{A}$}
		\STATE $P_{h}(s^{\dagger}|s,a)=1-\sum_{s^{\prime}\in\mathcal{S}:(h,s,a,s^{\prime})\notin \mathcal{F}}P_{h}(s^{\prime}|s,a).$
		\ENDFOR
		\FOR{$a\in \mathcal{A}$}
		\STATE $P_{h}(s^{\dagger}|s^{\dagger},a)=1.$
		\ENDFOR
		\STATE \textbf{Return} $P$.
	\end{algorithmic}
\end{algorithm}

Algorithm~\ref{algo_transition_kernel} receives a data set $\mathcal{D}$, a set $\mathcal{F}$ of infrequent tuples, a transition kernel $P$ and a target layer $h$ which we want to update. The goal is to update the $h$-th layer of the input transition kernel $P$ while the remaining layers stay unchanged. The construction of $P_{h}$ is for such tuples in $\mathcal{F}$, the transition kernel $P_{h}(s^{\prime}|s,a)$ is 0. For the states not in $\mathcal{F}$, $P_{h}(s^{\prime}|s,a)$ is the empirical estimate. At last, $P_{h}(s^{\dagger}|s,a)=1-\sum_{s^{\prime}\in\mathcal{S}:(h,s,a,s^{\prime})\notin \mathcal{F}}P_{h}(s^{\prime}|s,a)$ holds so that $P_{h}$ is a valid transition kernel. For a better understanding, the construction is similar to the construction of $\widetilde{P}$. We first let $P$ be the empirical estimate based on $\mathcal{D}$, then for $(h,s,a,s^{\prime})\in\mathcal{F}$, we move the probability of $P_{h}(s^{\prime}|s,a)$ to $P_{h}(s^{\dagger}|s,a)$.
\newpage
\section{Transition between original MDP and absorbing MDP}\label{appb2}
For any reward function $r$ defined on the original MDP $P$, we abuse the notation and use it on the absorbing version. We extend the definition as: \\
$$r(s,a) = \begin{cases} r(s,a), &   s\in \mathcal{S},  \\ 0,& s=s^{\dagger}. \end{cases}$$
For any policy $\pi$ defined on the original MDP $P$, we abuse the notation and use it on the absorbing version. We extend the definition as: \\
$$\pi(\cdot|s) = \begin{cases} \pi(\cdot|s), &  s\in \mathcal{S}, \\ \text{arbitrary distribution,} & s=s^{\dagger}. \end{cases}$$
Under this definition of $r$ and $\pi$, the expected reward under the absorbing MDP is fixed because once we enter the absorbing state $s^{\dagger}$, we will not get any more reward, so the policy at $s^{\dagger}$ has no influence on the value function.  For any policy $\pi$ defined under the absorbing MDP, we can directly apply it under the true MDP and analyze its value function because $\pi_{h}(\cdot |s)$ has definition for any $s\in \mathcal{S}$.

In this paper, $P$ is the real MDP, which is under original MDP. In each stage, $\widetilde{P}$ is an absorbing MDP constructed based on infrequent tuples $\mathcal{F}$ and the real MDP $P$. When we run the algorithm, we don't know the exact $\widetilde{P}$, but we know the intermediate transition kernel $P^{int}$, which is also an absorbing MDP. In Algorithm~\ref{algo2}, the $\widehat{P}$ we construct is the empirical estimate of $\widetilde{P}$, which is also an absorbing MDP. In the proof of this paper, a large part of discussion is under the framework of absorbing MDP. When we specify that the discussion is under absorbing MDP with absorbing state $s^{\dagger}$, any transition kernel $P^{\prime}$ satisfies $P^{\prime}_{h}(s^{\dagger}|s^{\dagger},a)=1$ for any $(a,h)\in\mathcal{A}\times[H]$. For the reward functions in this paper, they are all defined under original MDP, when applied under absorbing MDP, the transition rule follows what we just discussed.

\section{Technical lemmas}\label{appt}
\begin{lemma}[Bernstein's inequality]\label{lem3}
	Let $x_{1},\cdots,x_{n}$ be independent bounded random variables such that $\mathbb{E}[x_{i}]=0$ and $|x_{i}|\leq A$ with probability $1$. Let $\sigma^{2}=\frac{1}{n}\sum_{i=1}^{n}\mathrm{Var}[x_{i}]$, then with probability $1-\delta$ we have 
	$$|\frac{1}{n}\sum_{i=1}^{n}x_{i}|\leq \sqrt{\frac{2\sigma^{2}\log(2/\delta)}{n}}+\frac{2A}{3n}\log(2/\delta).$$
\end{lemma}

\begin{lemma}[Empirical Bernstein's inequality \citep{maurer2009empirical}]\label{lem4}
	Let $x_{1},\cdots,x_{n}$ be i.i.d random variables such that $|x_{i}|\leq A$ with probability $1$. Let $\overline{x}=\frac{1}{n}\sum_{i=1}^{n}x_{i}$,and $\widehat{V}_{n}=\frac{1}{n}\sum_{i=1}^{n}(x_{i}-\overline{x})^{2}$, then with probability $1-\delta$ we have 
	$$|\frac{1}{n}\sum_{i=1}^{n}x_{i}-\mathbb{E}[x]|\leq \sqrt{\frac{2\widehat{V}_{n}\log(2/\delta)}{n}}+\frac{7A}{3n}\log(2/\delta).$$
\end{lemma}

\begin{lemma}[Lemma F.4 in \citep{dann2017unifying}]\label{lem7}
Let $F_{i}$ for $i = 1\cdots$ be a filtration and $X_{1},\cdots, X_{n}$ be a sequence of Bernoulli random variables with $\mathbb{P}(X_{i} = 1|F_{i-1}) = P_{i}$ with $P_{i}$ being $F_{i-1}$-measurable and $X_{i}$ being $F_{i}$ measurable. It holds that
$$\mathbb{P}[\exists\, n: \sum_{t=1}^{n}X_{t} < \sum_{t=1}^{n}P_{t}/2-W ]\leq e^{-W}.$$	
\end{lemma}

\begin{lemma}\label{lem8}
	Let $F_{i}$ for $i = 1\cdots$ be a filtration and $X_{1},\cdots, X_{n}$ be a sequence of Bernoulli random variables with $\mathbb{P}(X_{i} = 1|F_{i-1}) = P_{i}$ with $P_{i}$ being $F_{i-1}$-measurable and $X_{i}$ being $F_{i}$ measurable. It holds that
	$$\mathbb{P}[\exists\, n: \sum_{t=1}^{n}X_{t} < \sum_{t=1}^{n}P_{t}/2-\iota ]\leq \frac{\delta}{HAK},$$	where $\iota=\log(2HAK/\delta)$.
\end{lemma}

\begin{proof}[Proof of Lemma~\ref{lem8}]
	Directly plug in $W=\iota$ in lemma~\ref{lem7}.
\end{proof}

\section{Proof of lemmas regarding Crude Exploration  (Algorithm~\ref{algo1})}\label{appa}
First, we want to highlight that in this paper, under the absorbing MDP, $\mathcal{S}$ only denotes the original states, the absorbing state $s^{\dagger}\notin \mathcal{S}$.

An upper bound for global switching cost is straightforward.
\begin{lemma}\label{lem1}
	The global switching cost of Algorithm~\ref{algo1} is bounded by $HSA$.
\end{lemma}

\begin{proof}[Proof of Lemma~\ref{lem1}]
	There are at most $HSA$ different $\pi_{h,s,a}$'s, Algorithm~\ref{algo1} will just run each policy for several episodes. 
\end{proof}

We can bound the difference between $\widetilde{P}$ and $P^{int}$ by empirical Bernstein's inequality (Lemma~\ref{lem4}).
\begin{lemma}\label{lem5}
	 Define the event $\mathcal{W}$ as: $\forall\, (h,s,a,s^{\prime})\in[H]\times \mathcal{S}\times\mathcal{A}\times\mathcal{S}$ such that $(h,s,a,s^{\prime})\notin \mathcal{F}$, $$|P^{int}_{h}(s^{\prime}|s,a)-\widetilde{P}_{h}(s^{\prime}|s,a)|\leq \sqrt{\frac{2P^{int}_{h}(s^{\prime}|s,a)\iota}{N_{h}(s,a)}}+\frac{7\iota}{3N_{h}(s,a)}.$$ Then with probability $1-\frac{S^{2}\delta}{K}$, the event $\mathcal{W}$ holds. In addition, we have that $\forall\, (h,s,a,s^{\prime})\in \mathcal{F}$, $$\widetilde{P}_{h}(s^{\prime}|s,a)=P^{int}_{h}(s^{\prime}|s,a)=0.$$
\end{lemma}

\begin{proof}[Proof of Lemma~\ref{lem5}]
	The first part is because of Lemma~\ref{lem4} and a union bound on all $(h,s,a,s^{\prime})\in [H]\times\mathcal{S}\times\mathcal{A}\times\mathcal{S}$. The second part is because of Definition~\ref{def2} and the definition of $P^{int}$ in Algorithm~\ref{algo1}.
\end{proof}

\begin{lemma}\label{lem6}
	Conditioned on the event $\mathcal{W}$ in Lemma~\ref{lem5}, $\forall\, (h,s,a,s^{\prime})\in[H]\times \mathcal{S}\times\mathcal{A}\times\mathcal{S}$ such that $(h,s,a,s^{\prime})\notin \mathcal{F}$, it holds that
	$$ (1-\frac{1}{H})P^{int}_{h}(s^{\prime}|s,a)\leq \widetilde{P}_{h}(s^{\prime}|s,a) \leq (1+\frac{1}{H})P^{int}_{h}(s^{\prime}|s,a). $$
\end{lemma}

\begin{proof}[Proof of Lemma~\ref{lem6}]
	Because the event $\mathcal{W}$ is true, we have $\forall\, (h,s,a,s^{\prime})\in[H]\times \mathcal{S}\times\mathcal{A}\times\mathcal{S}$ such that $(h,s,a,s^{\prime})\notin \mathcal{F}$, $$|P^{int}_{h}(s^{\prime}|s,a)-\widetilde{P}_{h}(s^{\prime}|s,a)|\leq \sqrt{\frac{2P^{int}_{h}(s^{\prime}|s,a)\iota}{N_{h}(s,a)}}+\frac{7\iota}{3N_{h}(s,a)}.$$
	By the definition of $\mathcal{F}$ that $\mathcal{F}= \{(h,s,a,s^{\prime})|N_{h}(s,a,s^{\prime})\leq C_{1}H^{2}\iota\}$, $\forall\, (h,s,a,s^{\prime})\in[H]\times \mathcal{S}\times\mathcal{A}\times\mathcal{S}$ such that $(h,s,a,s^{\prime})\notin \mathcal{F}$, $N_{h}(s,a,s^{\prime})\geq C_{1}H^{2}\iota$. \\
	Recall that for such $(h,s,a,s^{\prime})\notin\mathcal{F}$, $P^{int}_{h}(s^{\prime}|s,a)=\frac{N_{h}(s,a,s^{\prime})}{N_{h}(s,a)}$, we have
	\begin{align*}
	|P^{int}_{h}(s^{\prime}|s,a)-\widetilde{P}_{h}(s^{\prime}|s,a)|&\leq \sqrt{\frac{2\iota}{N_{h}(s,a,s^{\prime})}}P^{int}_{h}(s^{\prime}|s,a)+\frac{7P^{int}_{h}(s^{\prime}|s,a)\iota}{3N_{h}(s,a,s^{\prime})}\\ &\leq (\sqrt{\frac{2}{C_{1}}}+\frac{7}{3C_{1}H})\cdot\frac{1}{H}P^{int}_{h}(s^{\prime}|s,a) \\&\leq \frac{1}{H}P^{int}_{h}(s^{\prime}|s,a). 
	\end{align*}
	The first inequality is because of the definition of $P^{int}$. The second inequality is because of the definition of $\mathcal{F}$. The last inequality is because of the choice of $C_{1}=6$. \\
	Then the proof is completed by arranging $|P^{int}_{h}(s^{\prime}|s,a)-\widetilde{P}_{h}(s^{\prime}|s,a)|\leq \frac{1}{H}P^{int}_{h}(s^{\prime}|s,a)$.
\end{proof}

From Lemma~\ref{lem6}, we can see that for those tuples $(h,s,a,s^{\prime})$ not in $\mathcal{F}$, the estimate of the transition kernel satisfies $ (1-\frac{1}{H})P^{int}_{h}(s^{\prime}|s,a)\leq \widetilde{P}_{h}(s^{\prime}|s,a) \leq (1+\frac{1}{H})P^{int}_{h}(s^{\prime}|s,a)$ with high probability. In addition, for those states $(h,s,a,s^{\prime})\in \mathcal{F}$, $P^{int}_{h}(s^{\prime}|s,a)=\widetilde{P}_{h}(s^{\prime}|s,a)=0$, which means this inequality holds for all $(h,s,a,s^{\prime})\in [H]\times\mathcal{S}\times\mathcal{A}\times\mathcal{S}$. For simplicity, we use a new definition $\theta$-multiplicatively accurate to describe the relationship between $P^{int}$ and $\widetilde{P}$.

\begin{definition}[$\theta$-multiplicatively accurate for transition kernels (under absorbing MDP)]\label{def3}
	Under the absorbing MDP with absorbing state $s^{\dagger}$, a transition kernel $P^{\prime}$ is $\theta$-multiplicatively accurate to another transition kernel $P^{\prime\prime}$ if $$(1-\theta)P^{\prime}_{h}(s^{\prime}|s,a) \leq P^{\prime\prime}_{h}(s^{\prime}|s,a) \leq (1+\theta)P^{\prime}_{h}(s^{\prime}|s,a)$$ for all $(h,s,a,s^{\prime})\in[H]\times\mathcal{S}\times\mathcal{A}\times\mathcal{S}$ and there is no requirement for the case when $s^{\prime}=s^{\dagger}$.
\end{definition}

Because of Lemma~\ref{lem5} and Lemma~\ref{lem6}, we have that with probability $1-\frac{S^{2}\delta}{K}$, $P^{int}$ is $\frac{1}{H}$-multiplicatively accurate to $\widetilde{P}$. Next, we will compare the visitation probability of each state $(h,s,a)$ under two transition kernels that are close to each other. 

\begin{lemma}\label{lem2}
	Define $1_{h,s,a}$ to be the reward function $r^{\prime}$ such that $r^{\prime}_{h^{\prime}}(s^{\prime},a^{\prime})=\mathds{1}[(h^{\prime},s^{\prime},a^{\prime})=(h,s,a)]$. Similarly, define  $1_{h,s}$ to be the reward function $r^{\prime}$ such that $r^{\prime}_{h^{\prime}}(s^{\prime},a^{\prime})=\mathds{1}[(h^{\prime},s^{\prime})=(h,s)]$. Then $V^{\pi}(1_{h,s,a},P^{\prime})$ and $V^{\pi}(1_{h,s},P^{\prime})$ denote the visitation probability of $(h,s,a)$ and $(h,s)$, respectively, under $\pi$ and $P^{\prime}$. Under the absorbing MDP with absorbing state $s^{\dagger}$, if $P^{\prime}$ is $\frac{1}{H}$-multiplicatively accurate to $P^{\prime\prime}$, for any policy $\pi$ and any $(h,s,a)\in[H]\times\mathcal{S}\times\mathcal{A}$, it holds that
	$$\frac{1}{4}V^{\pi}(1_{h,s,a},P^{\prime}) \leq V^{\pi}(1_{h,s,a},P^{\prime\prime}) \leq 3V^{\pi}(1_{h,s,a},P^{\prime}).$$
\end{lemma}

\begin{proof}[Proof of Lemma~\ref{lem2}]
	Under the absorbing MDP, for any trajectory $\tau=\{s_{1},a_{1},\cdots,s_{h},a_{h}\}$ (truncated at time step $h$) such that $(s_{h},a_{h})=(s,a)$ and $s\in\mathcal{S}$, we have $s_{h^{\prime}}\neq s^{\dagger}$ for any $h^{\prime}\leq h-1$. Note that we only need to consider the trajectory truncated at time step $h$ because the visitation to $(h,s,a)$ only depends on this part of trajectory. We have for any truncated trajectory $\tau=\{s_{1},a_{1},\cdots,s_{h},a_{h}\}$ such that $(s_{h},a_{h})=(s,a)$, it holds that
	\begin{align*}
	\mathbb{P}_{\pi}[\tau|P^{\prime\prime}]&= \prod_{i=1}^{h}\pi_{i}(a_{i}|s_{i})\times \prod_{i=1}^{h-1}P_{i}^{\prime\prime}(s_{i+1}|s_{i},a_{i}) \\ &\leq (1+\frac{1}{H})^{H}\prod_{i=1}^{h}\pi_{i}(a_{i}|s_{i})\times \prod_{i=1}^{h-1}P_{i}^{\prime}(s_{i+1}|s_{i},a_{i}) \\ &\leq 3\mathbb{P}_{\pi}[\tau|P^{\prime}].
	\end{align*}
	The first inequality is because when $s_{i+1}\neq s^{\dagger}$, $(1-\frac{1}{H})P^{\prime}_{i}(s_{i+1}|s_{i},a_{i}) \leq P^{\prime\prime}_{i}(s_{i+1}|s_{i},a_{i}) \leq (1+\frac{1}{H})P^{\prime}_{i}(s_{i+1}|s_{i},a_{i})$.\\
	Let $\tau_{h,s,a}$ be the set of truncated trajectories such that $(s_{h},a_{h})=(s,a)$. Then
	$$V^{\pi}(1_{h,s,a},P^{\prime\prime})=\sum_{\tau\in \tau_{h,s,a}}\mathbb{P}_{\pi}[\tau|P^{\prime\prime}]\leq 3\sum_{\tau\in \tau_{h,s,a}}\mathbb{P}_{\pi}[\tau|P^{\prime}]=3V^{\pi}(1_{h,s,a},P^{\prime}).$$
	The left side of the inequality can be proven in a similar way, with $(1-\frac{1}{H})^{H}\geq \frac{1}{4}$ when $H\geq 2$.
\end{proof}

Under the absorbing MDP with absorbing state $s^{\dagger}$, we have shown (in Lemma~\ref{lem6}) that with high probability, for any $(h,s,a,s^{\prime})\in [H]\times \mathcal{S}\times \mathcal{A}\times \mathcal{S}$, it holds that
$ (1-\frac{1}{H})P^{int}_{h}(s^{\prime}|s,a)\leq \widetilde{P}_{h}(s^{\prime}|s,a) \leq (1+\frac{1}{H})P^{int}_{h}(s^{\prime}|s,a).$ Combined with Lemma~\ref{lem2}, we have with high probability, for any policy $\pi$ and any $(h,s,a)\in [H]\times \mathcal{S}\times \mathcal{A}$, 
$$\frac{1}{4}V^{\pi}(1_{h,s,a},P^{int}) \leq V^{\pi}(1_{h,s,a},\widetilde{P}) \leq 3V^{\pi}(1_{h,s,a},P^{int}).$$ Careful readers may find that for the visitation probability to the absorbing state $s^{\dagger}$, this inequality may not be true. However, this is not a problem for our propose, because we do not need to explore the absorbing state or consider the visitation probability of $s^{\dagger}$. 

The structure of the absorbing MDP also gives rise to the following lemma
 about the relationship between $\widetilde{P}$ and $P$.
\begin{lemma}\label{rem10}
 For any policy $\pi$ and any $(h,s,a)\in [H]\times \mathcal{S}\times \mathcal{A}$, $$V^{\pi}(1_{h,s,a},P)\geq V^{\pi}(1_{h,s,a},\widetilde{P}).$$
\end{lemma}

\begin{proof}[Proof of Lemma~\ref{rem10}]
    For any truncated trajectory $\tau$ that arrives at $(h,s,a)$ under $\widetilde{P}$, $\mathbb{P}_{\pi}[\tau|P]=\mathbb{P}_{\pi}[\tau|\widetilde{P}]$.
\end{proof}

Next, we will define the following bad event and explain the decomposition of its probability.

\begin{definition}[Bad event $\mathcal{B}$ and $\mathcal{B}_{h}$ under original MDP]\label{def4}
	For a trajectory $\{s_{1},a_{1},\cdots,s_{H},a_{H},s_{H+1}\}$ under original MDP and some policy, define $\mathcal{B}$ to be the event where there exists $h$ such that $(h,s_{h},a_{h},s_{h+1})\in \mathcal{F}$.
	Define $\mathcal{B}_{h}$, for $h=1,2,\cdots, H$ to be the event that $(h,s_{h},a_{h},s_{h+1})\in \mathcal{F}$ and $\ \forall\, h^{\prime}\leq h-1$, $(h^{\prime},s_{h^{\prime}},a_{h^{\prime}},s_{h^{\prime}+1})\notin \mathcal{F}$.
\end{definition}

We have that under the original MDP, $\mathcal{B}$ is the event that the trajectory finally enters $\mathcal{F}$ and $\mathcal{B}_{h}$ is the event that the trajectory first enters $\mathcal{F}$ at time step $h+1$.

\begin{definition}[Bad event $\mathcal{B}$ and $\mathcal{B}_{h}$ under absorbing MDP]\label{def40}
	For a trajectory $\{s_{1},a_{1},\cdots,s_{H},a_{H},s_{H+1}\}$ under absorbing MDP and some policy, define $\mathcal{B}$ to be the event where there exists $h$ such that $s_{h}=s^{\dagger}$.
	Define $\mathcal{B}_{h}$, for $h=1,2,\cdots, H$ to be the event that $s_{h+1}=s^{\dagger}$ and $\ \forall\, h^{\prime}\leq h$, $s_{h^{\prime}}\neq s^{\dagger}$.
\end{definition}

We have that under the absorbing MDP, $\mathcal{B}$ is the event that the trajectory finally enters $s^{\dagger}$ and $\mathcal{B}_{h}$ is the event that the trajectory first enters $s^{\dagger}$ at time step $h+1$. Note that under either the original MDP or the absorbing MDP, $\mathcal{B}$ is a disjoint union of $\mathcal{B}_{h}$ and $\mathbb{P}(\mathcal{B})=\sum_{h=1}^{H}\mathbb{P}(\mathcal{B}_{h})$.

Now we are ready to prove the key lemma about the difference between $\widetilde{P}$ and $P$. We will prove Lemma~\ref{lem9} and state an improved version under the special case where $\phi=\phi^{1}$.

\begin{lemma}[Restate Lemma~\ref{lem9}]\label{lemfirst}
	Conditioned on $\mathcal{W}$ in Lemma~\ref{lem5}, with probability $1-\frac{\delta}{AK}$, $\sup_{\pi\in\phi} \mathbb{P}_{\pi}[\mathcal{B}]\leq \frac{168S^{3}A^{2}H^{4}\iota}{T}$. 
\end{lemma}
\begin{proof}[Proof of Lemma~\ref{lemfirst}]
	We will prove that $\forall\, h\in [H]$, with probability $1-\frac{\delta}{HAK}$, $\sup_{\pi\in \phi} \mathbb{P}_{\pi}[\mathcal{B}_{h}]\leq \frac{168S^{3}A^{2}H^{3}\iota}{T}$. \\
	First, recall that the event $\mathcal{B}_{h}$ means $(h,s_{h},a_{h},s_{h+1})\in \mathcal{F}$ and $\ \forall\, h^{\prime}\leq h-1$, $(h^{\prime},s_{h^{\prime}},a_{h^{\prime}},s_{h^{\prime}+1})\notin \mathcal{F}$ under the original MDP. Also, $\mathcal{B}_{h}$ means that $s_{h+1}=s^{\dagger}$ and $\forall\, h^{\prime}\leq h$, $s_{h^{\prime}}\neq s^{\dagger}$ under the absorbing MDP. We have
	\begin{equation}\label{equ1}
	\begin{split}
	\mathbb{P}_{\pi}[\mathcal{B}_{h}|P]&=\sum_{\tau\in \mathcal{B}_{h}}\mathbb{P}_{\pi}[\tau|P] \\ &=\sum_{\tau_{:h+1}\in \mathcal{B}_{h}}\mathbb{P}_{\pi}[(s_{1},a_{1},\cdots,s_{h},a_{h})|P]P_{h}(s_{h+1}|s_{h},a_{h}) \\ &=\sum_{\tau_{:h+1}\in \mathcal{B}_{h}}\mathbb{P}_{\pi}[(s_{1},a_{1},\cdots,s_{h},a_{h})|\widetilde{P}]P_{h}(s_{h+1}|s_{h},a_{h}) \\ &=\sum_{s\in \mathcal{S},a}V^{\pi}(1_{h,s,a},\widetilde{P})\sum_{s^{\prime}\in\mathcal{S}:(h,s,a,s^{\prime})\in \mathcal{F}}P_{h}(s^{\prime}|s,a) \\ &=\sum_{s\in \mathcal{S},a}V^{\pi}(1_{h,s,a},\widetilde{P})\widetilde{P}_{h}(s^{\dagger}|s,a)\\ &=\mathbb{P}_{\pi}[\mathcal{B}_{h}|\widetilde{P}],
	\end{split}
	\end{equation}
	where $\tau_{:h+1}$ in line 2 and line 3 means the trajectory $\tau$ truncated at $s_{h+1}$. Note that we only need to consider the trajectory truncated at $s_{h+1}$ because the event $\mathcal{B}_{h}$ only depends on this part of trajectory. The $\mathcal{B}_{h}$ in the first three lines are defined under original MDP, while the $\mathcal{B}_{h}$ in the last line is defined under absorbing MDP. 
	The third equation is because for $(h,s,a,s^{\prime})\notin \mathcal{F}$, $P=\widetilde{P}$. The forth equation is because there is a bijection between trajectories that arrive at $(h,s,a)\in[H]\times\mathcal{S}\times\mathcal{A}$ under absorbing MDP and trajectories in $\mathcal{B}_{h}$ that arrive at the same $(h,s,a)$ under the original MDP. The fifth equation is because of the definition of $\widetilde{P}$. The last equation is because of the definition of $\mathcal{B}_{h}$ under the absorbing MDP $\widetilde{P}$.\\
	Recall that in Algorithm~\ref{algo1}, $\pi_{h,s,a}=\mathrm{argmax}_{\pi\in \phi}V^{\pi}(1_{h,s,a},P^{int}),$ then because of Lemma~\ref{lem2} and the fact that when constructing $\pi_{h,s,a}$, the first $h-1$ layers of $P^{int}$ is already same to the final output $P^{int}$, we have
	\begin{equation}\label{equ2}
	\begin{split}
	V^{\pi_{h,s,a}}(1_{h,s,a},\widetilde{P})&\geq \frac{1}{4}V^{\pi_{h,s,a}}(1_{h,s,a},P^{int}) \\ &= \frac{1}{4}\sup_{\pi\in \phi}V^{\pi}(1_{h,s,a},P^{int}) \\ &\geq \frac{1}{12}\sup_{\pi\in \phi}V^{\pi}(1_{h,s,a},\widetilde{P}),
	\end{split}
	\end{equation}
	where the two inequalities are because of Lemma~\ref{lem2}.\\
	Define $\pi_{h}$ to be a policy that chooses each $\pi_{h,s,a}$ with probability $\frac{1}{SA}$ for any $(s,a)\in \mathcal{S}\times \mathcal{A}$. Then we have 
	\begin{equation}\label{equ3}
	\begin{split}
	V^{\pi_{h}}(1_{h,s,a},\widetilde{P}) &\geq \frac{1}{SA}V^{\pi_{h,s,a}}(1_{h,s,a},\widetilde{P})\\ &\geq \frac{1}{12SA}\sup_{\pi\in \phi}V^{\pi}(1_{h,s,a},\widetilde{P}).
	\end{split}
	\end{equation}
	Note that in our Algorithm~\ref{algo1}, each policy $\pi_{h,s,a}$ will be run for $\frac{T}{HSA}$ episodes. Then for any event $\mathcal{E}$, we have that 
	$$\sum_{s,a}\frac{T}{HSA}\times \mathbb{P}_{\pi_{h,s,a}}[\mathcal{E}|\widetilde{P}]=\frac{T}{H}\times \mathbb{P}_{\pi_{h}}[\mathcal{E}|\widetilde{P}].$$
	We will assume that running each $\pi_{h,s,a}$ for $\frac{T}{HSA}$ episodes is equivalent to running $\pi_{h}$ for $\frac{T}{H}$ episodes because with Lemma~\ref{lem8}, we can derive the same lower bound for the total number of event $\mathcal{E}$. In the remaining part of the proof, we will analyze assuming we run $\pi_{h}$ for $\frac{T}{H} $ episodes.\\
	With the definition of $\pi_{h}$, we have
	\begin{align*}
	\mathbb{P}_{\pi_{h}}[\mathcal{B}_{h}|P]&=\mathbb{P}_{\pi_{h}}[\mathcal{B}_{h}|\widetilde{P}] \\ &=\sum_{s\in\mathcal{S},a}V^{\pi_{h}}(1_{h,s,a},\widetilde{P})\widetilde{P}_{h}(s^{\dagger}|s,a)\\ &\geq \frac{1}{12SA}\sum_{s\in\mathcal{S},a}\sup_{\pi\in\phi}V^{\pi}(1_{h,s,a},\widetilde{P})\widetilde{P}_{h}(s^{\dagger}|s,a)\\&\geq \frac{1}{12SA}\sup_{\pi\in\phi}\sum_{s\in\mathcal{S},a}V^{\pi}(1_{h,s,a},\widetilde{P})\widetilde{P}_{h}(s^{\dagger}|s,a) \\&=\frac{1}{12SA}\sup_{\pi\in\phi}\mathbb{P}_{\pi}[\mathcal{B}_{h}|\widetilde{P}]\\&=\frac{1}{12SA}\sup_{\pi\in\phi}\mathbb{P}_{\pi}[\mathcal{B}_{h}|P],
	\end{align*}
	where $\mathcal{B}_{h}|P$ is defined under original MDP while $\mathcal{B}_{h}|\widetilde{P}$ is defined under absorbing MDP. The first and the last equation is because of \eqref{equ1}. The first inequality is because of \eqref{equ3}. The second inequality is because the summation of maximum is larger than the maximum of summation.\\
	Suppose $\sup_{\pi\in\phi}\mathbb{P}_{\pi}[\mathcal{B}_{h}|P]\geq\frac{168S^{3}A^{2}H^{3}\iota}{T}$, then we have $$\mathbb{P}_{\pi_{h}}[\mathcal{B}_{h}|P]\geq\frac{1}{12SA}\sup_{\pi\in\phi}\mathbb{P}_{\pi}[\mathcal{B}_{h}|P]\geq \frac{14S^{2}AH^{3}\iota}{T}.$$
	Therefore, by Lemma~\ref{lem8}, with probability $1-\frac{\delta}{HAK}$, $\mathcal{B}_{h}$ occurs for at least $\frac{T}{H}\times \frac{14S^{2}AH^{3}\iota}{2T}-\iota > 6S^{2}AH^{2}\iota$ times during the exploration of the $h$-th layer. However, by the definition of $\mathcal{F}$, at each time step $h$, for each $(s,a,s^{\prime})$, the event $\mathcal{B}_{h}\cap \{(s_{h},a_{h},s_{h+1})=(s,a,s^{\prime})\}$ occurs for at most $6H^{2}\iota$ times, so the event $\mathcal{B}_{h}=\bigcup_{(s,a,s^{\prime})}(\mathcal{B}_{h}\cap \{(s_{h},a_{h},s_{h+1})=(s,a,s^{\prime})\})$ occurs for at most $6S^{2}AH^{2}\iota$ times in total, which leads to contradiction. \\
	As a result, we have $\forall\, h\in [H]$, with probability $1-\frac{\delta}{HAK}$, $\sup_{\pi\in \phi}\mathbb{P}_{\pi}[\mathcal{B}_{h}]\leq \frac{168S^{3}A^{2}H^{3}\iota}{T}$. Combining these $H$ results, because $\mathbb{P}_{\pi}[\mathcal{B}]=\sum_{h=1}^{H}\mathbb{P}_{\pi}[\mathcal{B}_{h}]$, we have with probability $1-\frac{\delta}{AK}$, 
	$$\sup_{\pi\in \phi}\mathbb{P}_{\pi}[\mathcal{B}]\leq \frac{168S^{3}A^{2}H^{4}\iota}{T}.$$
\end{proof}

The following Lemma~\ref{lemsecond} is an improved version of the previous Lemma~\ref{lemfirst} under the special case where $\phi=\phi^{1}$. When the policy set contains all the deterministic policies, we can have a bound with tighter dependence on $S$.
\begin{lemma}\label{lemsecond}
	Conditioned on $\mathcal{W}$ in Lemma~\ref{lem5}, if $\phi=\phi^{1}=\{\text{the set of all deterministic policies}\}$, with probability $1-\frac{2S\delta}{K}$, $\sup_{\pi\in\phi^{1}} \mathbb{P}_{\pi}[\mathcal{B}]\leq \frac{672S^{3}AH^{4}\iota}{T}$.
\end{lemma}
\begin{proof}[Proof of Lemma~\ref{lemsecond}]
	We will prove that $\forall\, h\in [H]$, with probability $1-\frac{2S\delta}{HK}$, $\sup_{\pi\in \phi^{1}} \mathbb{P}_{\pi}[\mathcal{B}_{h}]\leq \frac{672S^{3}AH^{3}\iota}{T}$. \\
	First, same to~\eqref{equ1} and~\eqref{equ2} we have
$$	\mathbb{P}_{\pi}[\mathcal{B}_{h}|P]=\mathbb{P}_{\pi}[\mathcal{B}_{h}|\widetilde{P}], $$
$$ V^{\pi_{h,s,a}}(1_{h,s,a},\widetilde{P})\geq \frac{1}{12}\sup_{\pi\in \phi^{1}}V^{\pi}(1_{h,s,a},\widetilde{P}).$$
Same to the proof of Lemma~\ref{lemfirst}, we define $\pi_{h}$ to be a policy that chooses each $\pi_{h,s,a}$ with probability $\frac{1}{SA}$ for any $(s,a)\in \mathcal{S}\times \mathcal{A}$. Then we have 
$$V^{\pi_{h}}(1_{h,s,a},\widetilde{P})\geq \frac{1}{12SA}\sup_{\pi\in \phi^{1}}V^{\pi}(1_{h,s,a},\widetilde{P})=\frac{1}{12SA}\sup_{\pi\in \phi^{1}}V^{\pi}(1_{h,s},\widetilde{P}).$$
The last equation is because $\phi^{1}$ consists of all the deterministic policies, for the optimal policy $\pi$ to visit $(h,s)$, we can just let $\pi_{h}(s)=a$ to construct a policy that can visit $(h,s,a)$ with the same probability. Similar to the proof of Lemma~\ref{lemfirst}, we can assume that running each $\pi_{h,s,a}$ for $\frac{T}{HSA}$ episodes is equivalent to running $\pi_{h} $ for $\frac{T}{H}$ episodes.\\
Because of Lemma~\ref{rem10}, we have $ V^{\pi_{h}}(1_{h,s,a},P)\geq V^{\pi_{h}}(1_{h,s,a},\widetilde{P})$. Also, there are $\frac{T}{H}$ episodes used to explore the $h$-th layer of the MDP, by Lemma~\ref{lem8} and a union bound, we have with probability $1-\frac{S\delta}{HK}$, for any $(s,a)\in \mathcal{S}\times\mathcal{A}$,  $$N_{h}(s,a)\geq \frac{T}{H}\times \frac{V^{\pi_{h}}(1_{h,s,a},P)}{2}-\iota\geq \frac{T}{H}\times \frac{V^{\pi_{h}}(1_{h,s,a},\widetilde{P})}{2}-\iota\geq \frac{T\cdot \sup_{\pi\in\phi^{1}}V^{\pi}(1_{h,s},\widetilde{P})}{24HSA}-\iota.$$ For fixed $(s,a)\in\mathcal{S}\times\mathcal{A}$, if $\frac{T\cdot \sup_{\pi\in\phi^{1}}V^{\pi}(1_{h,s},\widetilde{P})}{24HSA}\leq 2\iota$, we have that 
$$\sup_{\pi\in\phi^{1}}V^{\pi}(1_{h,s},\widetilde{P})\leq \frac{48HSA\iota}{T}.$$
Otherwise, $N_{h}(s,a)\geq \frac{T\cdot \sup_{\pi\in\phi^{1}}V^{\pi}(1_{h,s},\widetilde{P})}{48HSA}$. \\
By Lemma~\ref{lem8} and a union bound, we have with probability $1-\frac{S\delta}{HK}$, for any $(s,a)\in \mathcal{S}\times\mathcal{A}$, $$6SH^{2}\iota\geq \frac{N_{h}(s,a)\mathbb{P}[\mathcal{F}|(h,s,a)]}{2}-\iota,$$ where $\mathbb{P}[\mathcal{F}|(h,s,a)]$ is the conditional probability of entering $\mathcal{F}$ at time step $h$, given the state-action pair at time step $h$ is $(s,a)$. This is because similar to the proof of Lemma~\ref{lemfirst}, the event of entering $\mathcal{F}$ from $(h,s,a)$ happens for at most $6SH^{2}\iota$ times.\\
Then we have with probability $1-\frac{2S\delta}{HK}$, $\sup_{\pi\in\phi^{1}}V^{\pi}(1_{h,s},\widetilde{P})\leq \frac{48HSA\iota}{T}$ or $\mathbb{P}[\mathcal{F}|(h,s,a)]\leq\frac{672H^{3}S^{2}A\iota}{T\cdot \sup_{\pi\in\phi^{1}}V^{\pi}(1_{h,s},\widetilde{P})}$ holds for any $(s,a)\in \mathcal{S}\times\mathcal{A}$. Therefore, it holds that
\begin{align*}
\sup_{\pi\in\phi^{1}}\mathbb{P}_{\pi}[\mathcal{B}_{h}|P]&=\sup_{\pi\in\phi^{1}}\mathbb{P}_{\pi}[\mathcal{B}_{h}|\widetilde{P}] \\ &=\sup_{\pi\in\phi^{1}}\sum_{s\in\mathcal{S}}V^{\pi}(1_{h,s},\widetilde{P})\max_{a\in\mathcal{A}}\widetilde{P}_{h}(s^{\dagger}|s,a)\\ & =\sup_{\pi\in\phi^{1}}\sum_{s\in\mathcal{S}}V^{\pi}(1_{h,s},\widetilde{P})\max_{a\in\mathcal{A}}\mathbb{P}[\mathcal{F}|(h,s,a)]\\&\leq \sup_{\pi\in\phi^{1}}\sum_{s\in\mathcal{S}}\max\{\frac{48HSA\iota}{T},\frac{672H^{3}S^{2}A\iota}{T}\} \\&=\frac{672H^{3}S^{3}A\iota}{T}.
\end{align*}
The inequality is because if $\frac{T\cdot \sup_{\pi\in\phi^{1}}V^{\pi}(1_{h,s},\widetilde{P})}{24HSA}\leq 2\iota$, 
$$V^{\pi}(1_{h,s},\widetilde{P})\max_{a\in\mathcal{A}}\mathbb{P}[\mathcal{F}|(h,s,a)]\leq V^{\pi}(1_{h,s},\widetilde{P})\leq \sup_{\pi\in\phi^{1}}V^{\pi}(1_{h,s},\widetilde{P})\leq \frac{48HSA\iota}{T}.$$
Otherwise we have that 
$$V^{\pi}(1_{h,s},\widetilde{P})\max_{a\in\mathcal{A}}\mathbb{P}[\mathcal{F}|(h,s,a)]\leq V^{\pi}(1_{h,s},\widetilde{P})\frac{672H^{3}S^{2}A\iota}{T\cdot \sup_{\pi\in\phi^{1}}V^{\pi}(1_{h,s},\widetilde{P})}\leq \frac{672H^{3}S^{2}A\iota}{T}.$$
As a result, we have $\forall\, h\in [H]$, with probability $1-\frac{2S\delta}{HK}$, $\sup_{\pi\in \phi^{1}}\mathbb{P}_{\pi}[\mathcal{B}_{h}]\leq \frac{672S^{3}AH^{3}\iota}{T}$. Combining these $H$ results, because $\mathbb{P}_{\pi}[\mathcal{B}]=\sum_{h=1}^{H}\mathbb{P}_{\pi}[\mathcal{B}_{h}]$, we have with probability $1-\frac{2S\delta}{K}$, 
$$\sup_{\pi\in \phi^{1}}\mathbb{P}_{\pi}[\mathcal{B}]\leq \frac{672S^{3}AH^{4}\iota}{T}.$$
\end{proof}

\begin{remark}
We can see that with the same algorithm, the analysis is different for the general case and the special case that $\phi$ contains all the deterministic policies. The main technical reason is when $\phi=\phi^{1}$, $\sup_{\pi\in \phi^{1}}V^{\pi}(1_{h,s,a},\widetilde{P})=\sup_{\pi\in \phi^{1}}V^{\pi}(1_{h,s},\widetilde{P})$ while this does not hold for general policy set $\phi$. We will show that the part of the regret due to the construction of $\widetilde{P}$ is a lower order term. We prove a better bound for the case $\phi=\phi^{1}$ mainly for a better sample complexity in reward-free setting. 
\end{remark}

\begin{lemma}[Restate Lemma~\ref{lem10}]\label{lemthird}
	Conditioned on $\mathcal{W}$, under the high-probability event in Lemma~\ref{lemfirst} and Lemma~\ref{lemsecond}, we have for any policy $\pi\in\phi$ and reward function $r^{\prime}$, $$0\leq V^{\pi}(r^{\prime},P)-V^{\pi}(r^{\prime},\widetilde{P})\leq \frac{168S^{3}A^{2}H^{5}\iota}{T}.$$ In addition, if $\phi=\phi^{1}=\{\text{the set of all deterministic policies}\}$, we have that for any policy $\pi\in\phi^{1}$ and reward function $r^{\prime}$, $$0\leq V^{\pi}(r^{\prime},P)-V^{\pi}(r^{\prime},\widetilde{P})\leq \frac{672S^{3}AH^{5}\iota}{T}.$$
\end{lemma}

\begin{proof}[Proof of Lemma~\ref{lemthird}]
	We will only prove the first part, the second part is almost the same. Because $\widetilde{P}$ is the absorbing version of $P$, the left hand side is obvious. For the right hand side, we have that if $\pi\in\phi$,
\begin{align*}
V^{\pi}(r^{\prime},P)&=\sum_{\tau\in \mathcal{B}^{c}}r^{\prime}(\tau)P_{\pi}(\tau)+\sum_{\tau\in \mathcal{B}}r^{\prime}(\tau)P_{\pi}(\tau) \\
&=\sum_{\tau\in \mathcal{B}^{c}}r^{\prime}(\tau)\widetilde{P}_{\pi}(\tau)+\sum_{\tau\in \mathcal{B}}r^{\prime}(\tau)P_{\pi}(\tau) \\ &\leq V^{\pi}(r^{\prime},\widetilde{P})+\sum_{\tau\in \mathcal{B}}r^{\prime}(\tau)P_{\pi}(\tau)\\ &\leq V^{\pi}(r^{\prime},\widetilde{P})+\sum_{\tau\in \mathcal{B}}HP_{\pi}(\tau) \\ &\leq V^{\pi}(r^{\prime},\widetilde{P})+HP_{\pi}(\mathcal{B})\\ &\leq V^{\pi}(r^{\prime},\widetilde{P})+\frac{168S^{3}A^{2}H^{5}\iota}{T}.
\end{align*}
Note that all the $\mathcal{B}$ here are defined under original MDP. The second equation is because $\widetilde{P}=P$ when $\tau\in \mathcal{B}^c$. The first inequality is due to non-negative reward in $\mathcal{B}$ under $\widetilde{P}$. The last inequality follows from Lemma~\ref{lemfirst}.
\end{proof}


\section{Proof of lemmas regarding Fine Exploration  (Algorithm~\ref{algo2})}\label{appb}
We first state a conclusion about the global switching cost of Algorithm~\ref{algo2}.
\begin{lemma}\label{lem11}
	The global switching cost of Algorithm~\ref{algo2} is bounded by HSA.
\end{lemma}

\begin{proof}[Proof of Lemma~\ref{lem11}]
	There are at most $HSA$ different $\pi_{h,s,a}$'s, Algorithm~\ref{algo2} will just run each policy for several times. 
\end{proof}

\begin{lemma}[Simulation lemma \citep{dann2017unifying}]\label{lem12}
	For any two MDPs $M^{\prime}$ and $M^{\prime\prime}$ with rewards $r^{\prime}$ and $r^{\prime\prime}$ and transition probabilities $\mathcal{P}^{\prime}$ and $\mathcal{P}^{\prime\prime}$, the difference in values $V^{\prime}$, $V^{\prime\prime}$ with respect to the same policy $\pi$ can be written as 
	$$V_{h}^{\prime}(s)-V_{h}^{\prime\prime}(s)=\mathbb{E}_{M^{\prime\prime},\pi}[\sum_{i=h}^{H}[r_{i}^{\prime}(s_{i},a_{i})-r_{i}^{\prime\prime}(s_{i},a_{i})+(\mathbb{P}_{i}^{\prime}-\mathbb{P}_{i}^{\prime\prime})V_{i+1}^{\prime}(s_{i},a_{i})]|s_{h}=s].$$
\end{lemma}

Now we can prove that value functions under $\widetilde{P}$ and $\widehat{P}$ are close to each other.
\begin{lemma}[Restate Lemma~\ref{lem14}]\label{lemforth}
	Conditioned on the fact that $P^{int}$ is $\frac{1}{H}$-multiplicatively accurate to $\widetilde{P}$ (the case in Lemma~\ref{lem6}), with probability $1-\frac{T\delta}{2K}$, for any policy $\pi\in\phi$ and reward function $r^{\prime}$, $$|V^{\pi}(r^{\prime},\widehat{P})-V^{\pi}(r^{\prime},\widetilde{P})|=O(\sqrt{\frac{H^{5}S^{2}A\iota}{T}}).$$
\end{lemma}

\begin{proof}[Proof of Lemma~\ref{lemforth}]
	In this part of proof, note that the reward function $r^{\prime}$ is defined under the original MDP. When we transfer $r^{\prime}$ to be a reward function under the absorbing MDP, $r^{\prime}_{h}(s^{\dagger},a)=0$ for any $(h,a)\in[H]\times \mathcal{A}$.  Therefore, if $\widehat{V}$ denotes the value function of some policy $\pi$ under the MDP with reward function $r^{\prime}$ and transition kernel $\widehat{P}$, we have $\widehat{V}_{h}(s^{\dagger})=0$, for any $h\in [H]$. Because of simulation lemma (Lemma~\ref{lem12}), we have that $\forall\,\pi\in\phi$,
	\begin{equation}\label{equ4}
	|V^{\pi}(r^{\prime},\widehat{P})-V^{\pi}(r^{\prime},\widetilde{P})|\leq \mathbb{E}_{\widetilde{P}}^{\pi}\sum_{h=1}^{H}|(\widehat{P}_{h}-\widetilde{P}_{h})\cdot\widehat{V}^{\pi}_{h+1}|, \\
	\end{equation}
	where $\widehat{V}^{\pi}_{h}$ is the value function of policy $\pi$ under $\widehat{P}$ and $r^{\prime}$ at time $h$. The proof holds simultaneously for all reward function $r^{\prime}$, so we will omit $r^{\prime}$ for simplicity. \\
	Then we have
	\begin{align*}
	&\mathbb{E}_{\widetilde{P}}^{\pi}|(\widehat{P}_{h}-\widetilde{P}_{h})\cdot\widehat{V}^{\pi}_{h+1}|= \sum_{a,s\in \mathcal{S}}|(\widehat{\mathbb{P}}_{h}-\widetilde{\mathbb{P}}_{h})\widehat{V}_{h+1}^{\pi}(s,a)|V^{\pi}(1_{h,s,a},\widetilde{P}) \\ &\leq \sqrt{\sum_{a,s\in \mathcal{S}}|(\widehat{\mathbb{P}}_{h}-\widetilde{\mathbb{P}}_{h})\widehat{V}_{h+1}^{\pi}(s,a)|^{2}V^{\pi}(1_{h,s,a},\widetilde{P})} \\ &= \sqrt{\sum_{a,s\in \mathcal{S}}|(\widehat{\mathbb{P}}_{h}-\widetilde{\mathbb{P}}_{h})\widehat{V}_{h+1}^{\pi}(s,a)|^{2}V^{\pi}(1_{h,s,a},\widetilde{P})\mathds{1}\{a=\pi_{h}(s)\}}.
	\end{align*}
	The first equation is because if the trajectory arrives at the absorbing state $s^{\dagger}$ at time step $h$, then $\widehat{P}_{h}(s^{\prime}|s^{\dagger},a)=\widetilde{P}_{h}(s^{\prime}|s^{\dagger},a)$ for any $a,s^{\prime}$. The first inequality is because of Cauchy-Schwarz inequality. The last equation is because for any $\pi\in\phi$, $\pi$ is deterministic.\\
	Define $\pi_{random}$ to be a policy that chooses each $\pi_{h,s,a}$ with probability $\frac{1}{HSA}$ for any $(h,s,a)\in [H]\times\mathcal{S}\times \mathcal{A}$. Define $\mu_{h}(s,a)$ to be $\mu_{h}(s,a)=V^{\pi_{random}}(1_{h,s,a},\widetilde{P})$. Then similar to~\eqref{equ2} and~\eqref{equ3}, we have for any $(h,s,a)\in[H]\times\mathcal{S}\times\mathcal{A}$,
\begin{equation}\label{equ6}
\sup_{\pi\in\phi}\frac{V^{\pi}(1_{h,s,a},\widetilde{P})}{\mu_{h}(s,a)}\leq12HSA.
\end{equation}
	Plugging in this result into the previous inequality, we have 
	\begin{align*}
	\mathbb{E}_{\widetilde{P}}^{\pi}|(\widehat{P}_{h}-\widetilde{P}_{h})\cdot\widehat{V}^{\pi}_{h+1}|&\leq\sqrt{\sum_{a,s\in \mathcal{S}}|(\widehat{\mathbb{P}}_{h}-\widetilde{\mathbb{P}}_{h})\widehat{V}_{h+1}^{\pi}(s,a)|^{2}V^{\pi}(1_{h,s,a},\widetilde{P})\mathds{1}\{a=\pi_{h}(s)\}}\\ &\leq
	\sqrt{\sum_{a,s\in \mathcal{S}}|(\widehat{\mathbb{P}}_{h}-\widetilde{\mathbb{P}}_{h})\widehat{V}_{h+1}^{\pi}(s,a)|^{2}\cdot 12HSA\mu_{h}(s,a)\cdot \mathds{1}\{a=\pi_{h}(s)\}} \\&\leq \sqrt{12HSA\cdot \sup_{\nu:\mathcal{S}\rightarrow \mathcal{A}}\sum_{a,s\in \mathcal{S}}|(\widehat{\mathbb{P}}_{h}-\widetilde{\mathbb{P}}_{h})\widehat{V}_{h+1}^{\pi}(s,a)|^{2}\mu_{h}(s,a)\cdot \mathds{1}\{a=\nu(s)\}} \\&\leq \sqrt{12HSA\cdot \sup_{\nu:\mathcal{S}\rightarrow \mathcal{A}}\sum_{a,s\in \mathcal{S}}|(\widehat{\mathbb{P}}_{h}-\widetilde{\mathbb{P}}_{h})\widehat{V}_{h+1}^{\pi}(s,a)|^{2}\mu_{h}^{\prime}(s,a)\cdot \mathds{1}\{a=\nu(s)\}} \\&=\sqrt{12HSA\cdot \sup_{\nu:\mathcal{S}\rightarrow\mathcal{A}}\mathbb{E}_{\mu^{\prime}_{h}}|(\widehat{\mathbb{P}}_{h}-\widetilde{\mathbb{P}}_{h})\widehat{V}_{h+1}^{\pi}(s,a)|^{2}\cdot \mathds{1}\{a=\nu(s)\}} \\&\leq \sqrt{12HSA\cdot \sup_{G:\mathcal{S}\cup s^{\dagger}\rightarrow[0,H]}\sup_{\nu:\mathcal{S}\cup s^{\dagger}\rightarrow \mathcal{A}}\mathbb{E}_{\mu^{\prime}_{h}}|(\widehat{\mathbb{P}}_{h}-\widetilde{\mathbb{P}}_{h})G  (s,a)|^{2}\cdot \mathds{1}\{a=\nu(s)\}},
	\end{align*}
	where $\mu_{h}^{\prime}(s,a)=V^{\pi_{random}}(1_{h,s,a},P)$ is the distribution of the data. The fourth inequality is because $\mu_{h}^{\prime}(s,a)\geq\mu_{h}(s,a)$ for any $(h,s,a)\in [H]\times\mathcal{S}\times\mathcal{A}$ (because of  Lemma~\ref{rem10}). In the equation, we extend the definition of $\mu^{\prime}$ by letting $\mu^{\prime}(s^{\dagger},a)=0$ so that $\mu^{\prime}$ is a distribution on $\mathcal{S}\cup s^{\dagger}\times\mathcal{A}$. The last inequality is because $\widehat{V}^{\pi}_{h+1}$ is a function from $\mathcal{S}\cup s^{\dagger}$ to $[0,H]$.\\
	Note that our data follows the distribution $\mu^{\prime}$. In addition, from the definition of $\widetilde{P}$ and $\widehat{P}$, we have that $\widehat{P}$ is the empirical estimate of $\widetilde{P}$. 
	By Lemma~\ref{lem13} (which we state right after) we have with probability $1-\frac{T\delta}{2K}$, for any $h\in[H]$, policy $\pi\in\phi$ and reward function $r^{\prime}$,
	$$\mathbb{E}_{\widetilde{P}}^{\pi}|(\widehat{P}_{h}-\widetilde{P}_{h})\cdot\widehat{V}^{\pi}_{h+1}|=O(\sqrt{\frac{H^{3}S^{2}A\iota}{T}}).$$
	By~\eqref{equ4}, we have with probability $1-\frac{T\delta}{2K}$, for any policy $\pi\in\phi$ and reward function $r^{\prime}$,
	$$|V^{\pi}(r^{\prime},\widehat{P})-V^{\pi}(r^{\prime},\widetilde{P})|=O(\sqrt{\frac{H^{5}S^{2}A\iota}{T}}).$$
\end{proof}

\begin{lemma}[Lemma C.2 in \citep{jin2020reward}]\label{lem13}
	Suppose $\widehat{\mathbb{P}}$ is the empirical transition matrix formed by sampling according to $\mu$ distribution for $N$ samples, $G$ can be any function from $\mathcal{S}$ to $[0,H]$, $\nu$ can be any function from $\mathcal{S}$ to $\mathcal{A}$, then with probability at least $1-\delta$, we have for any $h\in[H]$:
\begin{align*}	
&\max_{G:\mathcal{S}\rightarrow[0,H]}\max_{\nu:\mathcal{S}\rightarrow\mathcal{A}}\mathbb{E}_{\mu_{h}}|(\widehat{\mathbb{P}}_{h}-\mathbb{P}_{h})G(s,a)|^{2}\mathds{1}\{a=\nu(s)\}\leq O(\frac{H^{2}S\log(HAN/\delta)}{N}).
\end{align*}
\end{lemma}
This is a critical lemma that requires delicate arguments to prove. For self-containedness, we include a full proof with more technical details.
\begin{proof}[Proof of Lemma~\ref{lem13}]
Define random variables
$$X=(\widehat{\mathbb{P}}_{h}G(s,a)-G(s^{\prime}))^{2}-(\mathbb{P}_{h}G(s,a)-G(s^{\prime}))^{2},$$
$$\overline{X}=(f(s,a)-G^{\prime}(s^{\prime}))^{2}-(g(s,a)-G^{\prime}(s^{\prime}))^{2},$$ 
$$Y=X\mathds{1}\{a=\nu(s)\},\,\overline{Y}=\overline{X}\mathds{1}\{a=\nu^{\prime}(s)\},$$ 
where $(s,a,s^{\prime})\sim\mu_{h}\times\mathbb{P}_{h}(\cdot|s,a)$.
 Here, $\nu^{\prime}$ can be any function from $\mathcal{S}$ to $\mathcal{A}$ (in equation \eqref{equ10}) and will be the same as $\nu$ in equation~\eqref{equ9},~\eqref{equ8},~\eqref{equ11}. Also, $f$, $g$ and $G^{\prime}$ can be any function (in equation~\eqref{equ10}) from the $\epsilon$-cover defined later and will be the closest function from the $\epsilon$-cover to $\widehat{\mathbb{P}}G(s,a)\mathds{1}\{a=\nu(s)\}$, $\mathbb{P}G(s,a)\mathds{1}\{a=\nu(s)\}$ and $G$ in equation~\eqref{equ9},~\eqref{equ8},~\eqref{equ11}.
Also, we define
$$X_{i}=(\widehat{\mathbb{P}}_{h}G(s_{i},a_{i})-G(s_{i}^{\prime}))^{2}-(\mathbb{P}_{h}G(s_{i},a_{i})-G(s_{i}^{\prime}))^{2},$$
$$\overline{X}_{i}=(f(s_{i},a_{i})-G^{\prime}(s_{i}^{\prime}))^{2}-(g(s_{i},a_{i})-G^{\prime}(s_{i}^{\prime}))^{2},$$ 
$$Y_{i}=X_{i}\mathds{1}\{a_{i}=\nu(s_{i})\},\,\overline{Y}_{i}=\overline{X}_{i}\mathds{1}\{a_{i}=\nu^{\prime}(s_{i})\},$$ where $(s_{i},a_{i},s_{i}^{\prime})$ is the i-th sample in time step $h$ we collect. Notice that for every tuple $(v^{\prime},f,g,G^{\prime})$ and $\overline{Y}$, $\overline{Y}_{i}$ related to this tuple, we have that $\overline{Y}_{i}$'s are i.i.d samples from the distribution of $\overline{Y}$.\\
Same to the proof in \citep{jin2020reward}, we have these three properties of $Y$ and $Y_{i}$. \\
$$(1)\,\mathbb{E}Y=\mathbb{E}_{\mu_{h}}|(\widehat{\mathbb{P}}_{h}-\mathbb{P}_{h})G(s,a)|^{2}\mathds{1}\{a=\nu(s)\}.$$
$$(2)\,\sum_{i=1}^{N}Y_{i}\leq 0.$$
$$(3)\,\mathrm{Var}\{Y\}\leq 4H^{2}\mathbb{E}(Y).$$
Since we are taking maximum over $\nu$ and $G(s)$, and $\widehat{\mathbb{P}}$ is random, we need to cover all the possible $\nu$ and all the possible values of $\widehat{\mathbb{P}}G(s,a)\mathds{1}\{a=\nu(s)\}$, $\mathbb{P}G(s,a)\mathds{1}\{a=\nu(s)\}$ and $G$ to $\epsilon$ accuracy to use Bernstein's inequality for the functions in the cover. For $\nu$, there are $A^{S}$ deterministic policies in total. Given a fixed $\nu$, $\widehat{\mathbb{P}}G(s,a)\mathds{1}\{a=\nu(s)\}$, $\mathbb{P}G(s,a)\mathds{1}\{a=\nu(s)\}$ and $G$ can be covered by $(H/\epsilon)^{3S}$ values because the first two functions can be covered by $(H/\epsilon)^{2S}$ values (for $a\neq \nu(s)$ the first two will be $0$) and $G$ itself can be covered by $(H/\epsilon)^{S}$ values. \\
By Bernstein's inequality (Lemma~\ref{lem3}) and a union bound, we have with probability $1-\delta$, for any $(\nu^{\prime},f,g,G^{\prime})$ in the $\epsilon$-cover, it holds that
\begin{equation}\label{equ10}
    \mathbb{E}(\overline{Y})-\frac{1}{N}\sum_{i=1}^{N}\overline{Y}_{i}\leq \sqrt{\frac{2\mathrm{Var}\{\overline{Y}\}\log((\frac{H}{\epsilon})^{3S}\cdot A^{S}\cdot \frac{H}{\delta})}{N}}+\frac{H^{2}\log((\frac{H}{\epsilon})^{3S}\cdot A^{S}\cdot \frac{H}{\delta})}{3N}.
\end{equation}
Under this high probability case, for any $G$, $\nu$ and $\widehat{\mathbb{P}}$, choose $\nu^{\prime}=\nu$ and $(f,g,G^{\prime})$ to be the closest function from the $\epsilon$-cover to  $\widehat{\mathbb{P}}G(s,a)\mathds{1}\{a=\nu(s)\}$, $\mathbb{P}G(s,a)\mathds{1}\{a=\nu(s)\}$ and $G$. Let $\overline{Y}$ be defined according to this $(\nu,f,g,G^{\prime})$. Then we have
\begin{equation}\label{equ9}
    |Y-\overline{Y}|\leq 4H\epsilon,
\end{equation}
for any possible $(s,a,s^{\prime})$ (hence this inequality is also true when adding expectation to both terms) and
\begin{equation}\label{equ8}
\begin{split}
    |\mathrm{Var}\{Y\}-\mathrm{Var}\{\overline{Y}\}|&\leq |\mathbb{E}(\overline{Y}^{2})-\mathbb{E}(Y^{2})|+|(\mathbb{E}\overline{Y})^{2}-(\mathbb{E}Y)^{2}| \\&\leq 2H^{2}\cdot 4H\epsilon+2H^{2}\cdot 4H\epsilon\\&=16H^{3}\epsilon.
\end{split}
\end{equation}
As a result, we have
\begin{equation}\label{equ11}
\begin{split}
&\mathbb{E}_{\mu_{h}}|(\widehat{\mathbb{P}}_{h}-\mathbb{P}_{h})G(s,a)|^{2}\mathds{1}\{a=\nu(s)\}=\mathbb{E}Y\\&\leq \mathbb{E}Y-\frac{1}{N}\sum_{i=1}^{N}Y_{i} \\&\leq \mathbb{E}\overline{Y}-\frac{1}{N}\sum_{i=1}^{N}\overline{Y}_{i}+8H\epsilon \\&\leq \sqrt{\frac{2\mathrm{Var}\{\overline{Y}\}\log((\frac{H}{\epsilon})^{3S}\cdot A^{S}\cdot \frac{H}{\delta})}{N}}+\frac{H^{2}\log((\frac{H}{\epsilon})^{3S}\cdot A^{S}\cdot \frac{H}{\delta})}{3N}+8H\epsilon \\&\leq \sqrt{\frac{2\mathrm{Var}\{Y\}\log((\frac{H}{\epsilon})^{3S}\cdot A^{S}\cdot \frac{H}{\delta})}{N}}+\frac{H^{2}\log((\frac{H}{\epsilon})^{3S}\cdot A^{S}\cdot \frac{H}{\delta})}{3N}\\&+8H\epsilon+\sqrt{\frac{2\cdot 16H^{3}\epsilon\log((\frac{H}{\epsilon})^{3S}\cdot A^{S}\cdot \frac{H}{\delta})}{N}}.
\end{split}
\end{equation}
The first inequality is because of property 2. The second inequality is because of \eqref{equ9}. The third inequality is because of \eqref{equ10}. The last inequality is because of \eqref{equ8} and $\sqrt{a+b}\leq \sqrt{a}+\sqrt{b}$.\\
We can simply choose $\epsilon=\frac{HS}{32N}$ and plug in property (3) of $Y$, then by solving the following quadratic inequality we can finish the proof.
\begin{equation}
\begin{split}
&\mathbb{E}_{\mu_{h}}|(\widehat{\mathbb{P}}_{h}-\mathbb{P}_{h})G(s,a)|^{2}\mathds{1}\{a=\nu(s)\}\\& \leq\sqrt{\frac{8H^{2}\mathbb{E}_{\mu_{h}}|(\widehat{\mathbb{P}}_{h}-\mathbb{P}_{h})G(s,a)|^{2}\mathds{1}\{a=\nu(s)\}\cdot 3S\log(\frac{32HAN}{\delta})}{N}}+\frac{2H^{2}S\log(\frac{32HAN}{\delta})}{N}.
\end{split}
\end{equation}
\end{proof}

\section{Proof of Theorem~\ref{the1}}\label{appc}
We first give a proof for the upper bound on the number of stages.
\begin{lemma}\label{lem50}
If $T^{(k)}=K^{1-\frac{1}{2^{k}}}$ for $k=1,2\cdots$, we have 
$$K_{0}=\min\{j:2\sum_{k=1}^{j}T^{(k)}\geq K\}=O(\log\log K).$$
\end{lemma}

\begin{proof}[Proof of Lemma~\ref{lem50}]
    Take $j=\log_{2}\log_{2}K$, we have $2T^{(j)}=2\frac{K}{K^{(\log_{2}K)^{-1}}}=2\frac{K}{2}=K$, which means that $$K_{0}\leq \log_{2}\log_{2}K+1=O(\log\log K).$$
\end{proof}

Then we are able to bound the total global switching cost of Algorithm~\ref{algo3}.
\begin{lemma}\label{lem15}
	The total global switching cost of Algorithm~\ref{algo3} is bounded by $O(HSA\log\log T)$.
\end{lemma}

\begin{proof}[Proof of Lemma~\ref{lem15}]
	For each stage, the global switching cost is bounded by $2HSA$ because of Lemma~\ref{lem1} and Lemma~\ref{lem11}. There are $O(\log\log T)$ stages, so the total global switching cost is at most $O(HSA\log\log T)$.
\end{proof}
 
 Recall that in Algorithm~\ref{algo3}, in each stage $k$ ($1\leq k\leq K_{0}$), we run Algorithm~\ref{algo1} to construct the infrequent tuples $\mathcal{F}^{k}$ and the intermediate transition kernel $P^{int,k}$. The absorbing MDP $\widetilde{P}^{k}$ is constructed as in Definition~\ref{def2} based on $\mathcal{F}^{k}$ and the real MDP $P$. Then we run Algorithm~\ref{algo2} to construct an empirical estimate of $\widetilde{P}^{k}$, which is $\widehat{P}^{k}$.
\begin{lemma}[Restate Lemma~\ref{lem16}]\label{lemfifth}
	There exists a constant $C$, such that with probability $1-\delta$, it holds that for any $k$ and $\pi\in\phi^{k}$,  
	$$|V^{\pi}(r,\widehat{P}^{k})-V^{\pi}(r,P)|\leq C(\sqrt{\frac{H^{5}S^{2}A\iota}{T^{(k)}}}+\frac{S^{3}A^{2}H^{5}\iota}{T^{(k)}}),$$
	where $\iota=\log(2HAK/\delta)$.
\end{lemma}	

\begin{proof}[Proof of Lemma~\ref{lemfifth}]
For the choice of the universal constant $C$, we can first let $C$ be the constant hidden by the big $O$ in Lemma~\ref{lemforth}. Then because of Lemma~\ref{lemthird}, we can choose $C=\max\{C,168\}$. Note that this $C$ is also used as the universal constant in the elimination step in Algorithm~\ref{algo3}. \\ 
Because of triangular inequality, we have for any $k$ and any $\pi\in\phi^{k}$,
\begin{equation}
\begin{split}
|V^{\pi}(r,\widehat{P}^{k})-V^{\pi}(r,P)|\leq |V^{\pi}(r,\widehat{P}^{k})-V^{\pi}(r,\widetilde{P}^{k})| +|V^{\pi}(r,\widetilde{P}^{k})-V^{\pi}(r,P)|.
\end{split}
\end{equation}
For any $k\leq K_{0}$, because of Lemma~\ref{lem6} and Lemma~\ref{lemthird}, we have with probability $1-(\frac{S^{2}\delta}{K}+\frac{\delta}{AK})$, it holds that for any $\pi\in\phi^{k}$, $|V^{\pi}(r,P)-V^{\pi}(r,\widetilde{P}^{k})|\leq C\frac{S^{3}A^{2}H^{5}}{T^{(k)}}$ while $P^{int,k}$ is $\frac{1}{H}$-multiplicatively accurate to $\widetilde{P}^{k}$. Conditioned on this case, because of Lemma~\ref{lemforth}, we have with probability $1-\frac{T^{(k)}\delta}{2K}$, for any $\pi\in\phi^{k}$, $|V^{\pi}(r,\widehat{P}^{k})-V^{\pi}(r,\widetilde{P}^{k})|\leq C\sqrt{\frac{H^{5}S^{2}A\iota}{T^{(k)}}}$. Combining these two results, we have for any $1\leq k\leq K_{0}$, with probability $1-(\frac{S^{2}\delta}{K}+\frac{\delta}{AK}+\frac{T^{(k)}\delta}{2K})$, for any $\pi\in\phi^{k}$,
$$|V^{\pi}(r,\widehat{P}^{k})-V^{\pi}(r,P)|\leq C(\sqrt{\frac{H^{5}S^{2}A\iota}{T^{(k)}}}+\frac{S^{3}A^{2}H^{5}\iota}{T^{(k)}}).$$
Finally, the proof is completed through a union bound on $k$ and the fact that the failure probability is bounded by $$(\frac{S^{2}\delta}{K}+\frac{\delta}{AK})\times O(\log\log K)+\sum_{k=1}^{K_{0}}\frac{T^{(k)}\delta}{2K}\leq \delta,$$ if $K\geq \widetilde{\Omega}(S^{2})$.
\end{proof}

\begin{lemma}\label{lem17}
	Conditioned on the same high probability event of Lemma~\ref{lemfifth}, the optimal policy $\pi^\star$ will never be eliminated, i.e., $\pi^\star \in \phi^k$ for $k=1,2,3,\cdots$.
\end{lemma}	

\begin{proof}[Proof of Lemma~\ref{lem17}]
	We will prove this lemma by induction. First, because $\phi^{1}$ contains all the deterministic policies, $\pi^\star\in\phi^{1}$. Assume $\pi^\star\in\phi^{k}$, then we have
	\begin{align*}
	&\sup_{\widehat{\pi}\in\phi^{k}}V^{\widehat{\pi}}(r,\widehat{P}^{k})-V^{\pi^\star}(r,\widehat{P}^{k})=V^{\widehat{\pi}^{k}}(r,\widehat{P}^{k})-V^{\pi^\star}(r,\widehat{P}^{k})\\ &\leq |V^{\widehat{\pi}^{k}}(r,\widehat{P}^{k})-V^{\widehat{\pi}^{k}}(r,P)|+V^{\widehat{\pi}^{k}}(r,P)-V^{\pi^\star}(r,P)+|V^{\pi^\star}(r,P)-V^{\pi^\star}(r,\widehat{P}^{k})|\\&\leq 2C(\sqrt{\frac{H^{5}S^{2}A\iota}{T^{(k)}}}+\frac{S^{3}A^{2}H^{5}\iota}{T^{(k)}}).
	\end{align*}
	The last inequality is because of Lemma~\ref{lemfifth} and $\pi^\star$ is the optimal policy.\\
	Then according to the elimination rule in Algorithm~\ref{algo3}, we have that $\pi^\star\in\phi^{k+1}$, which means the optimal policy $\pi^\star$ will never be eliminated.
\end{proof}

\begin{lemma}\label{lem18}
	Conditioned on the same high probability event of Lemma~\ref{lemfifth}, for any remaining policies, i.e., $\pi\in\phi^{k+1}$, we have that
	$$V^{\pi^\star}(r,P)-V^{\pi}(r,P)\leq 4C(\sqrt{\frac{H^{5}S^{2}A\iota}{T^{(k)}}}+\frac{S^{3}A^{2}H^{5}\iota}{T^{(k)}}).$$
\end{lemma}

\begin{proof}[Proof of Lemma~\ref{lem18}]
	For $\pi\in\phi^{k+1}$, because the optimal policy $\pi^\star$ will never be eliminated (Lemma~\ref{lem17}), we have that 
	\begin{equation}\label{equ5}
	V^{\pi^\star}(r,\widehat{P}^{k})-V^{\pi}(r,\widehat{P}^{k})\leq \sup_{\widehat{\pi}\in\phi^{k}}V^{\widehat{\pi}}(r,\widehat{P}^{k})-V^{\pi}(r,\widehat{P}^{k})\leq 2C(\sqrt{\frac{H^{5}S^{2}A\iota}{T^{(k)}}}+\frac{S^{3}A^{2}H^{5}\iota}{T^{(k)}}).
	\end{equation}
	Then we have
	\begin{align*}
	V^{\pi^\star}(r,P)-V^{\pi}(r,P)&\leq |V^{\pi^\star}(r,P)-V^{\pi^\star}(r,\widehat{P}^{k})|+V^{\pi^\star}(r,\widehat{P}^{k})-V^{\pi}(r,\widehat{P}^{k})+|V^{\pi}(r,\widehat{P}^{k})-V^{\pi}(r,P)|\\ &\leq 4C(\sqrt{\frac{H^{5}S^{2}A\iota}{T^{(k)}}}+\frac{S^{3}A^{2}H^{5}\iota}{T^{(k)}}).
	\end{align*}
	The last inequality is because of Lemma~\ref{lemfifth} and~\eqref{equ5} .
\end{proof}

\begin{lemma}\label{lem19}
	Conditioned on the same high probability event of Lemma~\ref{lemfifth}, if $K\geq\widetilde{\Omega}(S^{8}A^{6}H^{10})$, the total regret is at most $\widetilde{O}(\sqrt{H^{4}S^{2}AT})$.
\end{lemma}

\begin{proof}[Proof of Lemma~\ref{lem19}]
	The regret for the first stage (stage 1) is at most $2HT^{(1)}=O(HK^{\frac{1}{2}})$. \\
	For stage $k\geq2$, because of Lemma~\ref{lem18}, the policies we use (any $\pi\in\phi^{k}$) are at most $4C(\sqrt{\frac{H^{5}S^{2}A\iota}{T^{(k-1)}}}+\frac{S^{3}A^{2}H^{5}\iota}{T^{(k-1)}})$ sub-optimal, so the regret for the k-th stage ($2T^{(k)}$ episodes) is at most
	$8CT^{(k)}(\sqrt{\frac{H^{5}S^{2}A\iota}{T^{(k-1)}}}+\frac{S^{3}A^{2}H^{5}\iota}{T^{(k-1)}}).$ \\
	Adding up the regret for each stage , we have that the total regret is bounded by
	\begin{align*}
	\text{Regret}(T)&\leq 2HK^{\frac{1}{2}}+\sum_{k=2}^{K_{0}}8CT^{(k)}(\sqrt{\frac{H^{5}S^{2}A\iota}{T^{(k-1)}}}+\frac{S^{3}A^{2}H^{5}\iota}{T^{(k-1)}}) \\ &= O(HK^{\frac{1}{2}})+ O(\sqrt{H^{5}S^{2}AK\iota}\cdot \log\log K)+O(S^{3}A^{2}H^{5}K^{\frac{1}{4}}\iota)\\ &= O(\sqrt{H^{5}S^{2}AK\iota}\cdot \log\log K)+O(S^{3}A^{2}H^{5}K^{\frac{1}{4}}\iota) \\ &= \widetilde{O}(\sqrt{H^{4}S^{2}AT}),
	\end{align*}
where the last equality is because $K\geq\widetilde{\Omega}(S^{8}A^{6}H^{10})$.
\end{proof}

Then Theorem~\ref{the1} holds because of Lemma~\ref{lem15}, Lemma~\ref{lemfifth} and Lemma~\ref{lem19}.

\begin{corollary}[Transition to a PAC bound]
	Under the same assumption as Theorem~\ref{the1}, for any $\epsilon >0$, Algorithm~\ref{algo3} can output a stochastic policy $\widehat{\pi}$ such that with high probability, $$V_{1}^\star(s_{1})-V_{1}^{\widehat{\pi}}(s_{1})\leq \epsilon$$ after $K=\widetilde{O}(\frac{H^{5}S^{2}A}{\epsilon^{2}})$ episodes.
\end{corollary}

\begin{proof}
	By Theorem~\ref{the1}, we have that the regret is bounded by $\widetilde{O}(\sqrt{H^{4}S^{2}AT})$ with high probability, which means we have
	$$\sum_{k=1}^{K}V_{1}^\star(s_{1})-V_{1}^{\pi_{k}}(s_{1})\leq \widetilde{O}(\sqrt{H^{4}S^{2}AT}),$$ where $\pi_{k}$ is the policy used in episode $k$. Now define a stochastic policy $\widehat{\pi}$ as 
	$$\widehat{\pi}=\frac{1}{K}\sum_{k=1}^{K}\pi_{k}.$$
	Then we have
	$$\mathbb{E}[V_{1}^\star(s_{1})-V_{1}^{\widehat{\pi}}(s_{1})]=\frac{1}{K}\sum_{k=1}^{K}V_{1}^\star(s_{1})-V_{1}^{\pi_{k}}(s_{1})\leq \widetilde{O}(\sqrt{\frac{H^{5}S^{2}A}{K}}).$$
	By Markov inequality, we have with high probability that 
	$$V_{1}^\star(s_{1})-V_{1}^{\widehat{\pi}}(s_{1})\leq \widetilde{O}(\sqrt{\frac{H^{5}S^{2}A}{K}}).$$
	Taking $K=\widetilde{O}(\frac{H^{5}S^{2}A}{\epsilon^{2}})$ bounds the above by $\epsilon$.
\end{proof}

\begin{remark}
In addition to constructing a random policy, we can output any policy in the remaining policy set $\phi^{K_{0}+1}$. Because there are $K_{0}=O(\log\log K)$ stages in total, the maximal $T^{(k)}$ is larger than $\Omega(\frac{K}{\log\log K})$. According to Lemma~\ref{lem18}, for any $\pi\in\phi^{K_{0}+1}$, $$V^{\pi^\star}(r,P)-V^{\pi}(r,P)\leq 4C(\sqrt{\frac{H^{5}S^{2}A\iota}{T^{(k)}}}+\frac{S^{3}A^{2}H^{5}\iota}{T^{(k)}}),\,\forall\,k=1,2,\cdots,K_{0}.$$
Combining these two results, we have for any $\pi\in\phi^{K_{0}+1}$,
$$V^{\pi^\star}(r,P)-V^{\pi}(r,P)\leq \widetilde{O}(\sqrt{\frac{H^{5}S^{2}A}{K}}).$$
Then $K=\widetilde{O}(\frac{H^{5}S^{2}A}{\epsilon^{2}})$ bounds the above by $\epsilon$.
\end{remark}

\section{Proof of lower bounds (Theorem \ref{thm:lower_bound} and Theorem~\ref{the2})}\label{appd}

\begin{theorem}[Restate Theorem \ref{thm:lower_bound}]\label{thm:restate_lower_bound}
	If $S\leq A^{\frac{H}{2}}$, for any algorithm with near-optimal $\widetilde{O}^*(\sqrt{T})$ regret bound, the global switching cost is at least $\Omega(HSA\log\log T)$.
\end{theorem}

\begin{theorem}[Restate Theorem~\ref{the2}]\label{lemlower}
	If $S\leq A^{\frac{H}{2}}$, for any algorithm with sub-linear regret bound, the global switching cost is at least $\Omega(HSA)$.
\end{theorem}

First, we will state the high level idea of the proof and some related discussions. For the $\Omega(HSA\log\log T)$ and $\Omega(HSA)$ lower bounds, we construct a MDP to show that it is at least as difficult as multi-armed bandits with $\Omega(HSA)$ arms. Previously, \citet{bai2019provably} proved an $\Omega(HSA)$ lower bound for local switching cost, which can only imply an $\Omega(A)$ lower bound for global switching cost in the worst case. In \citep{huang2022towards}, the authors proved an $\Omega(dH)$ lower bound for global switching cost under linear setting for all algorithms with PAC bound. The same lower bound is derived in \citep{gao2021provably} for all  algorithms with sub-linear regret. However, for the MDPs constructed by both papers, the number of actions available at each state is not the same, which means we can not get an $\Omega(HSA)$ lower bound for global switching cost by directly plugging in $d=SH$. Finally, we state that both lower bounds we present are optimal. An $\Omega(HSA\log\log T)$ lower bound on global switching cost is optimal since this bound is matched by the upper bound of switching cost in Theorem \ref{the1}.  Also, an $\Omega(HSA)$ lower bound on global switching cost is the optimal result for any no-regret algorithms. This is because we can run our low adaptive reward-free exploration (Algorithm~\ref{algo4}) for $K^{\frac{2}{3}}$ episodes and run the policy $\widehat{\pi}^{r}$ where $r$ is the real reward function for the remaining episodes. It can be shown that the regret is of order $O(T^{\frac{2}{3}})$ and global switching cost is bounded by $2HSA$.

\begin{proof}[Proof of Theorem \ref{thm:restate_lower_bound} and Theorem~\ref{lemlower}]
In this part of proof, we will add a mild assumption to parameters $H,S,A$ by assuming that $S\leq A^{\frac{H}{2}}$. First we will show that under this assumption, a MDP with $S$ states, $A$ actions and horizon $H$ can be at least as difficult as a multi-armed bandit with $\Omega(HSA)$ arms. We will consider a MDP with deterministic transition kernel and a fixed initial state $s_{1}$, a special state $s^{\dagger}$ will be used as absorbing state. The state space $\mathcal{S}$ can be divided into $\mathcal{S}=\{s_{1},s_{2},\cdots,s_{S-1},s^{\dagger}\}$. The action space $\mathcal{A}$ can be divided into $\mathcal{A}=\{a_{1},a_{2},\cdots,a_{A}\}$. Then the construction of MDP can be divided into three parts.

\paragraph{Absorbing state}
For the absorbing state $s^{\dagger}$, for any action $a$ and $h\in[H]$, $P_{h}(s^{\dagger}|s^{\dagger},a)=1$ while $P_{h}(s|s^{\dagger},a)=0$ for any $s\neq s^{\dagger}$. The reward is defined as $r_{h}(s^{\dagger},a)=0$ for any $a,h\in[H]$.

\paragraph{The first several layers}
Let $H_{0}$ be the minimal positive integer such that $S\leq A^{H_{0}}$. By the assumption that $S\leq A^{\frac{H}{2}}$, we have $H_{0}\leq \frac{H}{2}$. Then we can use an $A$-armed tree structure to ensure that for any state $s\neq s^{\dagger}$, there exists a unique path starting from $s_{1}$ to arrive at $(H_{0}+1,s)$. Formally, for each $h\leq H_{0}$, $i\in [1,S-1]$, $j\in[1,A]$, $P_{h}(s_{A(i-1)+j}|s_{i},a_{j})=1$ (here if $A(i-1)+j\geq S$, $s_{A(i-1)+j}=s^{\dagger}$). By induction, we can see that for $s\neq s^{\dagger}$, there is a unique trajectory that arrives at $s$ at time step $H_{0}+1$. The reward $r_{h}$ is always $0$ when $h\leq H_{0}$.

\paragraph{The remaining layers}
For each state $s\neq s^{\dagger}$ and $H>h\geq H_{0}+1$, there exists a single action $a_{s,h}$ such that $P_{h}(s|s,a_{s,h})=1$ while $P_{h}(s^{\prime}|s,a_{s,h})=0$ for $s^{\prime}\neq s$. For any $a\neq a_{s,h}$, we have that $P_{h}(s^{\dagger}|s,a)=1$. The reward function  $r_{h}(s,a_{s,h})=0$ when the action is the $a_{s,h}$ that keeps the agent at $s$, and the reward for other actions $r_{h}(s,a)$ is unknown (can be non-zero) when $H>h\geq H_{0}+1$, $a\neq a_{s,h}$. For the last layer, for any state $s\neq s^{\dagger}$ and any action $a\in\mathcal{A}$, $P_{H}(s^{\dagger}|s,a)=1$. The reward $r_{H}(s,a)$ is unknown and can be non-zero.

We can see that under this MDP, there are two cases. The first one is for some policies, the agent arrives at $(H_{0}+1,s^{\dagger})$ with no reward, then such a trajectory will have total reward $0$. The second one is the agent arrives at some $(H_{0}+1,s)$ ($s\neq s^{\dagger}$) and finally arrives at $s^{\dagger}$ from the tuple $(h,s,a)$ ($H>h\geq H_{0}+1$ and $a\neq a_{s,h}$ or $h=H$) with total reward $r_{h}(s,a)$. Also, for any deterministic policy, the trajectory is fixed, like pulling an ``arm'' in bandit setting. Note that the total number of such ``arms'' with non-zero unknown reward is at least $(S-1)(A-1)\frac{H}{2}\geq \Omega(HSA)$. Even if the transition kernel is known to the agent, this MDP is still as difficult as a multi-armed bandit problem with $\Omega(HSA)$ arms. Then Theorem \ref{thm:restate_lower_bound} results from Lemma \ref{lem:loglog} while Theorem \ref{lemlower} holds because of the following Lemma~\ref{lemma:mab}.
\end{proof}

\begin{lemma}[Theorem 2 in \citep{simchi2019phase}]\label{lem:loglog}
    Under the $K$-armed bandits problem, there exists an absolute constant $C>0$ such that for all $K>1,S\geq0,T\geq 2K$ and for all policy $\pi$ with switching budget $S$, the regret satisfies
    $$R^{\pi}(K,T)\geq\frac{C}{\log T}\cdot K^{1-\frac{1}{2-2^{-q(S,K)-1}}}T^{\frac{1}{2-2^{-q(S,K)-1}}},$$
    where $q(S,K)=\lfloor \frac{S-1}{K-1} \rfloor$. This further implies that $\Omega(K\log\log T)$ switches are necessary for achieving $\widetilde{O}(\sqrt{T})$ regret bound.
\end{lemma}

\begin{lemma}[Switching cost lower bound under MAB setting]\label{lemma:mab}
For any algorithm with sub-linear regret bound under $K$-armed bandit problem, the switching cost is at least $\Omega(K)$.
\end{lemma}

\begin{proof}[Proof of Lemma~\ref{lemma:mab}]
The $K$-armed bandit problem can be described by a vector $\mu=[\mu_{1},\cdots,\mu_{K}]$, where $\mu_{k}$ is the mean reward of the $k$-th arm. We consider the following base problem and $K$ possible problems.
$$\mu_{\text{base}}=[0,0,\cdots,0],\, \mu_{\text{problem k}}=[\mu_{1},\cdots,\mu_{K},\, \text{where}\, \mu_{i}=\mathds{1}(i=k)],\,\forall\,k\in[K].$$
Note that under the base problem, all arms have reward $0$ while under problem $k$, only the $k$-th arm has reward $1$ and all other arms have reward $0$. We will prove that for any algorithm with switching cost bounded by $\frac{K}{2}-1$, even if the reward is deterministic, the regret bound can not be sub-linear.\\
For any algorithm $Alg$ with switching cost smaller than $\frac{K}{2}-1$, because the maximum is larger than the average, we only need to provide a lower bound for $R=\frac{1}{K}\sum_{k=1}^{K}\mathbb{E}_{Alg}[\text{Regret}(T)|\text{problem k}]$.\\
Consider using algorithm $Alg$ on the base problem, let $\mathcal{T}=\{\tau\}$ be all the possible strings of $\{a_{1},r_{1},\cdots,a_{T},r_{T}\}$, where $a_{i}$ is the index of the arm pulled at time $i$ and $r_{i}$ is its reward. Then we have $\sum_{\tau\in\mathcal{T}}\mathbb{P}_{Alg}[\tau|\text{base problem}]=1$. In addition, because of the restriction on $Alg$ that the switching cost is bounded by $\frac{K}{2}-1$, each $\tau$ can only pull at most $\frac{K}{2}$ arms. Then it holds that
\begin{equation}
\begin{split}
R&=\frac{1}{K}\sum_{k=1}^{K}\mathbb{E}_{Alg}[\text{Regret}(T)|\text{problem k}]\\ &=\frac{1}{K}\sum_{k=1}^{K}\sum_{\tau\in\mathcal{T}}[\text{Regret}(\tau)|\text{problem k}]\times \mathbb{P}_{Alg}[\tau|\text{problem k}]\\ &=\sum_{\tau\in\mathcal{T}}\frac{1}{K}\sum_{k=1}^{K}[\text{Regret}(\tau)|\text{problem k}]\times \mathbb{P}_{Alg}[\tau|\text{problem k}]\\
&\geq \sum_{\tau\in\mathcal{T}}
\frac{T\times \mathbb{P}_{Alg}[\tau|\text{base problem}]}{2}\\&=\frac{T}{2},
\end{split}
\end{equation}
where $[\text{Regret}(\tau)|\text{problem k}]$ is the regret of string $\tau$ under problem $k$. The inequality is because for any string $\tau$, if problem k satisfies that the $k$-th arm does not appear in $\tau$, then $[\text{Regret}(\tau)|\text{problem k}]=T$ and $\mathbb{P}_{Alg}[\tau|\text{problem k}]=\mathbb{P}_{Alg}[\tau|\text{base problem}]$.
\end{proof}

\section{Proof of Theorem~\ref{the3}}\label{appe}
\begin{lemma}\label{lem20}
	The total global switching cost of Algorithm~\ref{algo4} is bounded by $2HSA$.
\end{lemma}

\begin{remark}
The switching cost $O(HSA)$ is optimal for reward-free setting. With the same MDP in Appendix~\ref{appd}, we can show that it takes $\Omega(HSA)$ switching cost to find out the mapping between the policy and which ``arm'' it is pulling.
\end{remark}

 In this part of proof, we also construct the absorbing MDP $\widetilde{P}$ in the same way as Definition~\ref{def2} based on the infrequent tuples $\mathcal{F}$ and the true MDP $P$. Similar to the proof of Lemma~\ref{lemfifth}, we construct the universal constant $C$, we can first let $C$ be the constant hidden by the big $O$ in Lemma~\ref{lemforth}. Then because of Lemma~\ref{lemthird}, we can choose $C=\max\{C,672\}$. With the $C$, we are able to choose the universal constant in this section to be $c^{\prime}=16C^{2}$. Note that the $\iota$ used in Algorithm~\ref{algo4} is $\iota=\log(\frac{2HA(N_{0}+N)}{\delta})$, while the $\iota^{\prime}$ used in the upper bound of sample complexity is $\iota^{\prime}=\log(\frac{HSA}{\epsilon\delta})$. According to the conclusion about Algorithm~\ref{algo1}, we have the following lemma regarding the choice of $N_{0}$.
\begin{lemma}\label{lem21}
	There exists $c^{\prime}>0$, for any $\epsilon >0$, when $N_{0}>c^{\prime}\cdot\frac{S^{3}AH^{5}\iota}{\epsilon}$, it holds that with probability $1-\frac{\delta}{2}$, for any $\pi\in\phi^{1}$ and any reward function $r$, $|V^{\pi}(r,P)-V^{\pi}(r,\widetilde{P})|\leq\frac{\epsilon}{4}$, while $P^{int}$ is $\frac{1}{H}$-multiplicatively accurate to $\widetilde{P}$.
\end{lemma}

\begin{proof}[Proof of Lemma~\ref{lem21}]
	This is a direct corollary of lemma~\ref{lem6} and lemma~\ref{lemthird}.
\end{proof}

Then because of the conclusions about Algorithm~\ref{algo2}, we have the following lemma regarding the choice of $N$.
\begin{lemma}\label{lem22}
	There exists $c^{\prime}>0$, for any $\epsilon >0$, when $N>c^{\prime}\cdot \frac{H^{5}S^{2}A\iota}{\epsilon^{2}}$, conditioned on the case in Lemma~\ref{lem21} that $P^{int}$ is $\frac{1}{H}$-multiplicatively accurate to $\widetilde{P}$, with probability $1-\frac{\delta}{2}$, for any $\pi\in\phi^{1}$ and reward function $r$, $|V^{\pi}(r,\widehat{P})-V^{\pi}(r,\widetilde{P})|\leq\frac{\epsilon}{4}$.
\end{lemma}

\begin{proof}[Proof of Lemma~\ref{lem22}]
	This is a direct corollary of Lemma~\ref{lemforth}.
\end{proof}

\begin{lemma}\label{lem23}
	There exists $c^{\prime}>0$, for any $\epsilon >0$, when the number of total episodes $K>c^{\prime}\cdot (\frac{H^{5}S^{2}A\iota}{\epsilon^{2}}+\frac{S^{3}AH^{5}\iota}{\epsilon})$, there exists a choice of $N_{0}$ and $N$ such that $N_{0}+N=K$ and with probability $1-\delta$, for any $\pi\in\phi^{1}$ and reward function $r$, $|V^{\pi}(r,\widehat{P})-V^{\pi}(r,P)|\leq\frac{\epsilon}{2}$.
\end{lemma}

\begin{proof}[Proof of Lemma~\ref{lem23}]
Because of triangular inequality, $$|V^{\pi}(r,\widehat{P})-V^{\pi}(r,P)|\leq|V^{\pi}(r,\widehat{P})-V^{\pi}(r,\widetilde{P})|+|V^{\pi}(r,\widetilde{P})-V^{\pi}(r,P)|.$$ Because of Lemma~\ref{lem21} and Lemma~\ref{lem22}, if we choose $N_{0}>c^{\prime}\cdot\frac{S^{3}AH^{5}\iota}{\epsilon}$ and $N>c^{\prime}\cdot \frac{H^{5}S^{2}A\iota}{\epsilon^{2}}$, it holds that with probability $1-\delta$, for any $\pi\in\phi^{1}$ and reward function $r$, 
$$|V^{\pi}(r,P)-V^{\pi}(r,\widetilde{P})|\leq\frac{\epsilon}{4},\,|V^{\pi}(r,\widehat{P})-V^{\pi}(r,\widetilde{P})|\leq\frac{\epsilon}{4}.$$
Then the proof is finished by plugging in these two inequalities to the triangular inequality.
\end{proof}

Note that $K>c^{\prime}\cdot (\frac{H^{5}S^{2}A\iota}{\epsilon^{2}}+\frac{S^{3}AH^{5}\iota}{\epsilon})$ is not a good representation of constraints on $K$ because $K$ appears on both sides of the inequality. The following Lemma~\ref{lemfinal} gives a valid solution to this inequality.

\begin{lemma}\label{lemfinal}
   For a fixed $c^{\prime}$, there exists a constant $c$ such that if $K>c(\frac{H^{5}S^{2}A\iota^{\prime}}{\epsilon^{2}}+\frac{S^{3}AH^{5}\iota^{\prime}}{\epsilon})$, then $K>c^{\prime}\cdot (\frac{H^{5}S^{2}A\iota}{\epsilon^{2}}+\frac{S^{3}AH^{5}\iota}{\epsilon})$, where $\iota=\log(\frac{2HAK}{\delta})$ and $\iota^{\prime}=\log(\frac{HSA}{\epsilon\delta})$.
\end{lemma}

\begin{lemma}\label{lem24}
	Conditioned on the case in Lemma~\ref{lem23}, we have that $0\leq V^{\pi^\star}(r,P)-V^{\widehat{\pi}^{r}}(r,P)\leq \epsilon$ holds for any reward function $r$.
\end{lemma}

\begin{proof}[Proof of Lemma~\ref{lem24}]
	$V^{\pi^\star}(r,P)-V^{\widehat{\pi}^{r}}(r,P)\geq 0$ directly results from the definition of optimal policy. Also,
	\begin{align*}
	V^{\pi^\star}(r,P)-V^{\widehat{\pi}^{r}}(r,P)&\leq |V^{\pi^\star}(r,P)-V^{\pi^\star}(r,\widehat{P})|+V^{\pi^\star}(r,\widehat{P})-V^{\widehat{\pi}^{r}}(r,\widehat{P})+|V^{\widehat{\pi}^{r}}(r,\widehat{P})-V^{\widehat{\pi}^{r}}(r,P)|\\ &\leq \epsilon.
	\end{align*}
	The second inequality is because of Lemma~\ref{lem23} and $\widehat{\pi}^{r}=\mathrm{argmax}_{\pi\in\phi^{1}}V^{\pi}(r,\widehat{P})$.
\end{proof}

Then Theorem~\ref{the3} holds because of Lemma~\ref{lem20}, Lemma~\ref{lem23}, Lemma~\ref{lemfinal} and Lemma~\ref{lem24}.

\section{Improved algorithm with near optimal number of batches}\label{appi}
First of all, we will analyze the batch complexity of APEVE (Algorithm~\ref{algo3}). Note that one application of Crude Exploration (Algorithm~\ref{algo1}) can be finished in $H$ batches, because the exploration of each layer can be finished in one batch. Besides, one application of Fine Exploration (Algorithm~\ref{algo2}) can be finished in one batch. According to the schedule of APEVE (Algorithm~\ref{algo3}), there are $O(\log\log T)$ stages, each stage contains one Crude Exploration and one Fine Exploration. Therefore, the batch complexity of APEVE is $O(H\log\log T)$.

We further improve the batch complexity by revising APEVE (Algorithm~\ref{algo3}) slightly to get this APEVE+ (Algorithm~\ref{algo6}). The main difference is that in Algorithm~\ref{algo6}, only the first two stages contain the use of Crude Exploration (Algorithm~\ref{algo1}), the remaining stages only run Fine Exploration (Algorithm~\ref{algo2}) and eliminate policies at the end of each stage. First, we consider the number of batches.  In Algorithm~\ref{algo6}, there are two applications of Crude Exploration and $O(\log\log T)$ applications of Fine Exploration, so we have the following theorem.

\begin{theorem}\label{the77}
APEVE+ (Algorithm~\ref{algo6}) can be applied in $O(2H+\log\log T)=O(H+\log\log T)$ batches.
\end{theorem}

\begin{remark}
\citet{gao2019batched} proved that for any algorithm with $\widetilde{O}(\sqrt{T})$ regret bound under multi-armed bandit setting, the number of batches is at least $\Omega(\log\log T)$. In the construction of our lower bound in Appendix~\ref{appd}, we show that tabular RL can be at least as difficult as a multi-armed bandit problem, which means the $\Omega(\log\log T)$ lower bound on batches also applies to tabular RL. Theorem B.3 in \citep{huang2022towards} states an $\Omega(\frac{H}{\log T})$ lower bound for number of batches for any algorithm with PAC guarantee. Because regret guarantee is stronger than PAC guarantee, this lower bound also applies to any algorithm with $\widetilde{O}(\sqrt{T})$ regret bound. Combining these two results, we have an $\Omega(\frac{H}{\log T}+\log\log T)$ lower bound on number of batches for any algorithm with $\widetilde{O}(\sqrt{T})$ regret. According to Theorem~\ref{the77}, we conclude that our Algorithm~\ref{algo6} nearly matches the lower bound of batches.
\end{remark}

Now we will consider the regret bound of APEVE+. We have the following key lemma whose proof and choice of the constant $C$ is identical to Lemma~\ref{lemfifth}.
\begin{lemma}
There exists a constant $C$, such that with probability $1-\delta$, for any $k=1,2$ and $\pi\in\phi^{k}$, 
	$$|V^{\pi}(r,\widehat{P}^{k})-V^{\pi}(r,P)|\leq C(\sqrt{\frac{H^{5}S^{2}A\iota}{T^{(k)}}}+\frac{S^{3}A^{2}H^{5}\iota}{T^{(k)}}),$$
	while for any $k\geq 3$ and $\pi\in\phi^{k}$,  $$|V^{\pi}(r,\widehat{P}^{k})-V^{\pi}(r,P)|\leq C(\sqrt{\frac{H^{5}S^{2}A\iota}{T^{(k)}}}+\frac{S^{3}A^{2}H^{5}\iota}{T^{(2)}}).$$
\end{lemma}

Then using the identical proof as Lemma~\ref{lem17} and Lemma~\ref{lem18}, we have an upper bound on the sub-optimality of policies used in each stage.

\begin{lemma}\label{lemj4}
	With probability $1-\delta$, for the policies that have not been eliminated at stage $k$, i.e. $\pi\in\phi^{k+1}$, we have that
	$$V^{\pi^\star}(r,P)-V^{\pi}(r,P)\leq 4C(\sqrt{\frac{H^{5}S^{2}A\iota}{T^{(k)}}}+\frac{S^{3}A^{2}H^{5}\iota}{T^{(k)}}),\ k=1,2.$$
	$$V^{\pi^\star}(r,P)-V^{\pi}(r,P)\leq 4C(\sqrt{\frac{H^{5}S^{2}A\iota}{T^{(k)}}}+\frac{S^{3}A^{2}H^{5}\iota}{T^{(2)}}),\ k\geq 3.$$
\end{lemma}

Now we are ready to bound the total regret of Algorithm~\ref{algo6}.
\begin{theorem}\label{lem:APEVE+}
If $K\geq\widetilde{\Omega}(S^{8}A^{6}H^{10})$, with probability $1-\delta$, the total regret of APEVE+ (Algorithm~\ref{algo6}) is at most $\widetilde{O}(\sqrt{H^{4}S^{2}AT})$.
\end{theorem}

\begin{proof}[Proof of Theorem~\ref{lem:APEVE+}]
	Because the case in Lemma~\ref{lemj4} holds with probability $1-\delta$, we will prove under the case in Lemma~\ref{lemj4} and show that the regret is $\widetilde{O}(\sqrt{H^{4}S^{2}AT})$. \\ The regret for the first stage (stage 1) is at most $2HT^{(1)}=O(HK^{\frac{1}{2}})$. \\
    Because of Lemma~\ref{lemj4}, the regret for the second stage (stage 2) is at most $2T^{(2)}\times4C(\sqrt{\frac{H^{5}S^{2}A\iota}{T^{(1)}}}+\frac{S^{3}A^{2}H^{5}}{T^{(1)}})$. \\
	For stage $k\geq3$, the policies we use (any $\pi\in\phi^{k}$) are at most $4C(\sqrt{\frac{H^{5}S^{2}A\iota}{T^{(k-1)}}}+\frac{S^{3}A^{2}H^{5}\iota}{T^{(2)}})$ sub-optimal, so the regret for the k-th stage is at most
	$8CT^{(k)}(\sqrt{\frac{H^{5}S^{2}A\iota}{T^{(k-1)}}}+\frac{S^{3}A^{2}H^{5}\iota}{T^{(2)}}).$ \\
	Adding up the regret for each stage , we have that the total regret is bounded by
	\begin{align*}
	\text{Regret}(T)&\leq 2HK^{\frac{1}{2}}+\sum_{k=2}^{K_{0}}8CT^{(k)}\sqrt{\frac{H^{5}S^{2}A\iota}{T^{(k-1)}}}+8CT^{(2)}\frac{S^{3}A^{2}H^{5}\iota}{T^{(1)}}+\sum_{k=3}^{K_{0}}8CT^{(k)}\frac{S^{3}A^{2}H^{5}\iota}{T^{(2)}} \\ &= O(HK^{\frac{1}{2}})+ O(\sqrt{H^{5}S^{2}AK\iota}\cdot \log\log K)+O(S^{3}A^{2}H^{5}K^{\frac{1}{4}}\iota)\\ &= O(\sqrt{H^{5}S^{2}AK\iota}\cdot \log\log K)+O(S^{3}A^{2}H^{5}K^{\frac{1}{4}}\iota) \\ &= \widetilde{O}(\sqrt{H^{4}S^{2}AT}),
	\end{align*}
where the last equality is because $K\geq\widetilde{\Omega}(S^{8}A^{6}H^{10})$.
\end{proof}

\begin{algorithm}
	\caption{Adaptive Policy Elimination by Value Estimation+ (APEVE+)}\label{algo6}
	\begin{algorithmic}[1]
		\STATE \textbf{Require}: Number of episodes for exploration $K$, $r$ is the known deterministic reward. Universal constant $C$. Failure probability $\delta$.
		\STATE \textbf{Initialize}: $T^{(k)}=K^{1-\frac{1}{2^{k}}}$, $k\leq K_{0}=O(\log\log K)$, $\phi^{1}:=\{ \text{the set of all the deterministic policies}\}$, $\iota=\log(2HAK/\delta)$. 
		\FOR{$k=1,2,\cdots,K_{0}$}  
		\STATE \maroon{$\diamond$ Number of episodes in $k$-th stage:}
		\IF{$T^{(1)}+T^{(2)}+\sum_{i=1}^{k}T^{(i)}\geq K$} 
		\STATE $T^{(k)}=K-T^{(1)}-T^{(2)}-\sum_{i=1}^{k-1}T^{(i)}$. (o.w. $T^{(k)}=K^{1-\frac{1}{2^{k}}}$)
		\ENDIF
		\ENDFOR
		\STATE \maroon{Adaptive policy elimination for the first two stages:}
		\FOR{$k=1,2$}
		\STATE \maroon{Update the infrequent set $\mathcal{F}^{k}$ and construct empirical estimate of the absorbing MDP:}
		\STATE $\mathcal{F}^{k}$,$P^{int,k}$ = Crude Exploration$(\phi^{k},T^{(k)}).$
		\STATE $\widehat{P}^{k}$ = Fine Exploration$(\mathcal{F}^{k},P^{int,k},T^{(k)},\phi^{k}).$
		\STATE $U^k=\emptyset$ 
		\FOR{$\pi\in\phi^{k}$}
		\IF{$V^{\pi}(r,\widehat{P}^{k})\leq sup_{\widehat{\pi}\in\phi^{k}}V^{\widehat{\pi}}(r,\widehat{P}^{k})-2C(\sqrt{\frac{H^{5}S^{2}A\iota}{T^{(k)}}}+\frac{S^{3}A^{2}H^{5}\iota}{T^{(k)}})$}
		\STATE Update $U^{k}\leftarrow U^k\cup \{\pi\}$.
		\ENDIF
		\ENDFOR
		\STATE $\phi^{k+1}\leftarrow\phi^{k}\backslash U^k$.
		\ENDFOR
		\STATE \maroon{Adaptive policy elimination for the remaining stages:}
		\FOR{$k=3,4,\cdots,K_{0}$}
		\STATE \maroon{Keep the infrequent set $\mathcal{F}^{2}$ and construct empirical estimate of the absorbing MDP with new data set:}
		\STATE $\widehat{P}^{k}$ = Fine Exploration$(\mathcal{F}^{2},P^{int,2},T^{(k)},\phi^{k}).$
		\STATE $U^k=\emptyset$ 
		\FOR{$\pi\in\phi^{k}$}
		\IF{$V^{\pi}(r,\widehat{P}^{k})\leq sup_{\widehat{\pi}\in\phi^{k}}V^{\widehat{\pi}}(r,\widehat{P}^{k})-2C(\sqrt{\frac{H^{5}S^{2}A\iota}{T^{(k)}}}+\frac{S^{3}A^{2}H^{5}\iota}{T^{(2)}})$}
		\STATE Update $U^{k}\leftarrow U^k\cup \{\pi\}$.
		\ENDIF
		\ENDFOR
		\STATE $\phi^{k+1}\leftarrow\phi^{k}\backslash U^k$.
		\ENDFOR
	\end{algorithmic}
\end{algorithm}

\end{onecolumn}

\end{document}